\documentclass[journal]{IEEEtran}
\usepackage{cite}
\usepackage{amsmath,amssymb,amsfonts}
\usepackage{algorithm,algpseudocode}
\newtheorem{theorem}{Theorem}
\newtheorem{lemma}{Lemma}
\newtheorem{definition}{Definition}
\newtheorem{proposition}{Proposition}
\newtheorem{remark}{Remark}

\usepackage{graphicx}
\usepackage{textcomp}
\usepackage{xcolor}
\usepackage{makecell}
\usepackage{adjustbox}
\usepackage{subfigure}
\usepackage{multirow}
\usepackage[font=small,aboveskip=2pt,belowskip=-5pt]{caption}

\DeclareMathOperator*{\argmin}{arg\,min}
\graphicspath{ {./figure/} }
\begin{document}

\title{
Large-Dimensional Multibody Dynamics Simulation Using Contact Nodalization and Diagonalization
}
\author{Jeongmin Lee, Minji Lee, and Dongjun Lee
\thanks{This research was supported by the Industrial Strategic Technology Development Program (20001045) of the Ministry of Trade, Industry \& Energy (MOTIE) of Korea; and the Engineering Research Center Program for Soft Robotics (2016R1A5A1938472) of the National Research Foundation (NRF) funded by the Ministry of Science, ICT \& Future Planning (MSIP) of Korea.}
\thanks{The authors are with the Department of Mechanical \& Aerospace Engineering and IAMD, Seoul National University, Seoul, Republic of Korea.
\{ljmlgh,mingg8,djlee\}@snu.ac.kr. Corresponding author: Dongjun Lee.}
}
\maketitle

\begin{abstract}
We propose a novel multibody dynamics simulation framework that can efficiently deal with large-dimensionality and complementarity multi-contact conditions.
Typical contact simulation approaches require performing contact impulse fixed-point iteration (I-FPI), which has high time-complexity from large-size matrix factorization and multiplication, as well as susceptibility to ill-conditioned contact situations.
To circumvent this, we propose a novel framework based on velocity fixed-point iteration (V-FPI), which, by utilizing a certain surrogate dynamics and contact nodalization (with virtual nodes), we achieve not only inter-contact decoupling but also their inter-axes decoupling (i.e., contact diagonalization) at each iteration step. 
This then enables us to one-shot/parallel-solve the contact problem during each V-FPI iteration-loop, while avoiding large-size/dense matrix inversion/multiplication, thereby, significantly speeding up the simulation time with improved convergence property.
We theoretically show that the solution of our framework is consistent with that of the original problem and, further, elucidate mathematical conditions for the convergence of our proposed solver.
Performance and properties of our proposed simulation framework are also demonstrated and experimentally-validated for various large-dimensional/multi-contact scenarios including deformable objects.
\end{abstract}

\section{Introduction}
\IEEEPARstart{A}{s} technology and automation advance, robots performing diverse tasks in various environments on behalf of humans are becoming an increasingly essential topic \cite{kunze18ral}.
Robots no longer do specific tasks in a stationary and well-known environment, necessitating the ability to notice changes in their surroundings and react immediately through interaction.
In this regard, enhancing online performance and flexibility through offline pre-training has recently emerged as a very promising solution.
These concepts have been implemented in a variety of ways (e.g, self-supervised learning \cite{narang21icra,lee20tro}, reinforcement learning \cite{agostinelli19nmi,chebotar19icra}) and have been used to succeed in tasks in a variety of challenging situations (e.g., climbing in complex terrain \cite{hwangbo19sr}, tight tolerance assembly \cite{son20iros}).

Since the approaches described above can be broadly understood in terms of optimization, they commonly require large amounts of high-quality data to achieve adequate performance and robustness.
These data can also be gathered in real-world contexts, but by using virtual environments (i.e., simulation), the data collecting process becomes considerably safer and faster, and the environment can be adjusted without the need for human intervention. 
Furthermore, perfect knowledge of the system can be used for more efficient learning \cite{pintorss18,johns21arxiv} and does not necessitate a separate estimating approach to collect data. As a result, simulation is frequently utilized as a backbone to train robots to do challenging tasks (i.e., sim-to-real transfer \cite{abbeel18icra,matas18corl,ken21icra}).

However, for practically useful data collection, simulation must be both fast and accurate.
One of the most challenging issues in dynamic simulation of robot applications is contact \cite{hofer20rss}, while it inevitably occurs in situations where the robot interacts with the environment and other objects. 
In this paper, we mainly focus on following common form of discrete-time dynamics equations with contact, as in \cite{todorov12iros,duriez18tro,daviet20tog}:
\begin{equation} \label{eq-lindyn}
    \begin{aligned}
        A_k\hat{v}_k = b_k + J_{c,k}^T\lambda_{c,k}
    \end{aligned}
\end{equation}
where $A_k\in\mathbb{R}^{n\times n},b_k\in\mathbb{R}^n$ are constructed from the system state at the $k$-th time step, $\hat{v}_k\in\mathbb{R}^n$ is the representative velocity between the time intervals of $k$-th to $k+1$-th steps, $J_{c,k}\in\mathbb{R}^{3n_c\times n}$ is the contact Jacobian, $\lambda_{c,k}\in\mathbb{R}^{3n_c}$ is the contact impulse, $n$ is the system dimension, and $n_c$ is the contact number at the time step.

In general, the contact constraints can be expressed as the complementarity-based relation between $\hat{v}_k$ and $\lambda_{c,k}$. Therefore, in many cases, the solution process of \eqref{eq-lindyn} requires the following contact impulse space transform:
\begin{equation} \label{eq-delassus}
    \begin{aligned}
        &J_{c,k}\hat{v}_k = J_{c,k}A_k^{-1}J_{c,k}^T\lambda_{c,k} + J_{c,k}A_k^{-1}b_k
    \end{aligned}
\end{equation}
where $J_{c,k}A_k^{-1}J_{c,k}^T$ is called Delassus operator \cite{horak19ral}.
This impulse-based formulation \eqref{eq-delassus} implies the linear relation between contact impulse and the velocity relative to the contact frame. Based on this formulation, various algorithms \cite{hwangbo18ral,daviet11tog,liu05tase,erleben17sca} perform impulse fixed-point iteration (I-FPI) to find a proper collision response that satisfies the contact conditions. 
Despite their popularity in rigid body simulation, I-FPI suffers from the following challenges in general multibody simulation:
\begin{itemize}
    \item Due to numerous coupling between states (e.g., finite element method), structure of the matrix $A_k$ in \eqref{eq-lindyn} is complex and time-varying, which makes the computation of Delassus operator in \eqref{eq-delassus} intractable especially for a large-size dynamics with varying contact points.
    \item In the case of ill-conditioned contact situations (i.e., the condition number of the Delassus operator is poor), the I-FPI adopted to \eqref{eq-delassus} converges slowly and often fails, resulting in implausible contact behavior of the simulation.
\end{itemize}
Accordingly, it is difficult to efficiently simulate a high degree of freedom (DOF) systems with multiple contacts, particularly in scenarios involving deformable objects.

I-FPI based approaches \cite{hwangbo18ral,daviet11tog,liu05tase,erleben17sca} may be characterized as executing fixed-point iteration to meet contact complementarity condition while intrinsically reflecting \eqref{eq-lindyn}.
In contrast, in this paper, we present a new method that uses fixed-point iteration to satisfy \eqref{eq-lindyn} while keeping the contact condition based on velocity fixed-point iteration (V-FPI).
We first propose contact nodalization, which turns all contacts into nodal situations (i.e., contact with Cartesian nodal points) while precisely preserving the contact condition. 
We then proceed to develop a novel numerical solver based on contact diagonalization, which is achieved through multiple contact generation and solution with respect to certain surrogate dynamics.
Each surrogate dynamics problem can be solved in a one-shot/parallelized manner due to the diagonal properties and converges to the original dynamics equation \eqref{eq-lindyn} as iteration progresses.
The main features of our framework can be summarized as follow:
\begin{itemize}
    \item $\textit{Scalable}$: 
    The entire procedure of the solver consists solely of matrix-vector multiplication and simple algebraic operations, resulting in low time and memory complexity.
    \item $\textit{Accurate}$: 
    Dynamics and contact conditions are precisely enforced, resulting in accurate contact simulation as demonstrated by experimental validations.
    \item $\textit{Convergent}$: 
    Convergence is theoretically investigated and practically proved to have fast and robust convergence even for ill-conditioned problems.
    \item $\textit{Versatile}$: 
    Diverse formulations including a combination of rigid-deformable body and maximal-generalized coordinate can be dealt with. Furthermore, it is compatible with various integrator types and friction models as well.
\end{itemize}
From now on, we refer to the framework as simulation using contact nodalization and diagonalization (COND).
We also release an implementation of COND for a specific scenario (cable winding manipulation, see Sec.~\ref{Sec-result} for detail) so that it can be utilized for simulation and learning benchmarks: https://github.com/INRoL/inrol\_sim\_cablewinding.

The problem of solving dynamics with contact has been studied for a long time.
One of the well-known directions is to use the spring-damper-based penalty contact force formulation \cite{yamane06icra,khude13jcnd}, which calculates and applies a force proportional to the penetration or sliding velocity.
By explicitly formulating the contact force, it is advantageous in terms of scalability.
However it often demands a very small time step to prevent penetration or physically odd behavior, and the result tends to vary depending on the gain value \cite{kaufman08tog,drumwright11icra,khude13jcnd}.

Another very traditional method is to use a direct solution of the linear complementarity problem (LCP) \cite{potra97nd}, which is based on a polyhedral shape approximation of the friction cone.
In contrast to the non-linear complementarity problem (NCP), LCP can be solved exactly using various algorithms such as \cite{llyod05icra}, but they require impulse space conversion and have a high computational load during the solution process.
In addition, because of the linearized friction cone, unplausible frictional behavior can be generated \cite{horak19ral}.

In modern simulation research, numerical methods are mainly used.
One of the most prevalent methods is the projected Gauss-Seidel (PGS) algorithm \cite{horak19ral}, which solves \eqref{eq-delassus} using Gauss-Seidel type I-FPI while projecting the solution to the friction cone.
Similarly, gradient-based methods such as \cite{mazhar15tog} can be used, with numerical acceleration schemes.
The methods are faster than direct methods in many cases and able to handle more preferable contact formulation (than LCP) including NCP and its convex relaxation \cite{tasora10coa,todorov14icra}.
Therefore, they have been widely adopted in open-source simulation software (e.g., Bullet \cite{bullet}, MuJoCo \cite{todorov12iros}, Chrono \cite{tasora15chrono}, RaiSim \cite{RaiSim}).
The projection step can sometimes be replaced with a slight modification using the bisection \cite{hwangbo18ral} and Newton \cite{daviet11tog} algorithms.
However, under a complex structure of $A$, there is still a scalability issue in multibody systems due to their reliance on impulse space conversion and their ability to converge slowly in ill-conditioned contact situations.

Various techniques have been proposed to improve computational efficiency for large-scale problems.
For soft objects, model order reduction is utilized to reduce the system dimension in \cite{yoon19iros,duriez18tro} with the open-source framework SOFA \cite{sofa}.
The methods are promising, yet the applicable scenarios are restricted since they assume small deformation or limited modes of system behavior.
Position-based dynamics (PBD) idea \cite{macklin16game,macklin19tog} is prevalent in graphics area, and also utilized in open-source software FleX \cite{flex} and Brax \cite{brax} with various robotic researches \cite{liu21icra,lie22arxiv}.
The main limitation of PBD for robotic simulation is their slow convergence to adequate accuracy in engineering, and possible occurrence of implausible contact behavior which is also described in \cite{daviet20tog}.  
Also, its non-linear Gauss-Seidel fashion constraint resolution is incompatible with generalized coordinates representation which is widely used in the robotic simulation.
Subsystem-based architectures are presented in \cite{peiret20ral,lee21icra} which split the original problem into smaller size problems with parallelization, but the methods are yet limited to systems with few interconnections between subsystems.
In \cite{daviet20tog} and \cite{li18tog}, efficient solvers that can simulate objects such as cloth and hair are proposed while avoiding impulse space conversion. 
In contrast to our framework, they rely on the global relaxation and linear solving process and cannot deal with general rigid body representation.
Another notable work is \cite{macklin19tog} that develops a non-smooth Newton method to solve contact conditions using Schur-complement and complementarity preconditioner. 
It has the advantage of fast convergence due to the nature of the second-order but requires multiple large-size linear solving processes for a single time step.

The rest of the paper is organized as follows. Sec.~\ref{Sec-dyn} will explain how we create the constrained dynamics with contact in a discrete-time domain (i.e., \eqref{eq-lindyn}). 
Then in Sec.~\ref{Sec-ics}, our main algorithm COND for multibody dynamics solver will be described in detail. 
Sec.~\ref{Sec-conv} will investigate the convergence of the solver. 
The simulation and experimental analysis to evaluate the performance of COND will be presented in Sec.~\ref{Sec-result}. 
Sec.~\ref{sec-discussion} provides some discussions on the framework and future work.
Finally, the concluding remark will be presented in Sec.~\ref{sec-conclusion}.

\section{Constrained Dynamics with Contact} \label{Sec-dyn}

In this section, we will describe how we construct the dynamic equation with multiple constraints and contact.

\subsection{Dynamics Integration} \label{slindyn}
We formulate discrete-time domain dynamics at the velocity and impulse level, as is generally the case to avoid inconsistency \cite{champneys16jam}.
Dynamics integration methods can be broadly classified according to explicit/implicit type and linear/non-linear type.
In this work, we construct the dynamic equation in linearized form, considering the constraints in an implicit manner. This is because:
1) while targeting complex systems including flexible bodies, the explicit method has clear limitations in its stability, and 2) non-linear integration can actually be expressed as an iteration of a linear integration.
The equation of motion of mechanical system under contacts in the continuous-time domain can be written as follow \cite{todorov12iros}:
\begin{equation} 
    \begin{aligned}
        M(q)\ddot{q} + C(q,\dot{q})\dot{q} + d\psi^T &= f_{ext} + J_c(q)^T\lambda_c \\
        &= f_{ext} + \sum_{m=1}^{n_c} J_{c,m}(q)^T\lambda_{c,m}
    \end{aligned}
\end{equation}
where $q$ is the generalized coordinate variable\footnote{can involve representation in Euclidean space as well as orientation representation such as SO(3). Note that generalized velocity $\dot{q}$ and acceleration $\ddot{q}$ can still be expressed as $\mathbb{R}^n$.} of system, $M(q),C(q,\dot{q})\in\mathbb{R}^{n\times n}$ are the mass, Coriolis matrix, $d\psi^T\in\mathbb{R}^n$ is the potential action, $f_{ext}\in\mathbb{R}^n$ is the external force, $\lambda_c=[\lambda_{c,1};\cdots;\lambda_{c,{n_c}}]\in\mathbb{R}^{3n_c}$ is the contact impulse with $\lambda_{c,m}=[\lambda_{n,m};\lambda_{t_1,m};\lambda_{t_2,m}]\in\mathbb{R}^{3}$, $J_c=[J_{c,1};\cdots;J_{c,{n_c}}]\in\mathbb{R}^{3n_c\times n}$ is the contact Jacobian, and subscripts $n,t_1,t_2$ denote the normal and tangential directions.
Then we perform the discretization of the dynamics as
\begin{equation} \label{disc_dyn}
    \begin{aligned}
        &M_k\frac{v_{k+1}-v_k}{t_k} + C_k v_k + d\psi_k^T = f_{ext,k} + J_{c,k}^T\lambda_{c,k} \\
        &\hat{v}_k=\frac{v_k+v_{k+1}}{2}, \quad q_{k+1} \leftarrow \text{update}(q_k,\hat{v}_k,t_k)
    \end{aligned}
\end{equation}
where $k$ denotes the time step index, $M_k=M(q_k)$, $C_k=C(q_k,v_k)$, $t_k$ is the step size, $v_k\in\mathbb{R}^n$ is the velocity, and $\hat{v}_k$ is the representative velocity of each time step. 
We derive potential action from the passivity relation presented in \cite{kim17ijrr} i.e.,
\begin{align} \label{eq-passive}
    d\psi_k^T\hat{v}_k t_k 
    &=\frac{\partial\psi}{\partial q_k}\hat{v}_k t_k + \frac{1}{2}\hat{v}_k^T\frac{\partial^2\psi}{\partial q_k^2}\hat{v}_k t_k^2 \\
    &\approx \psi_{k+1}-\psi_k
\end{align}
with the second-order approximation of exact potential energy deviation. 
Here, potential function may be non-convex therefore the Hessian term may not be symmetric positive definite.
Yet, some common approximations can be adopted to solve the issue in a compact way.
Consider the following constraint potential form:
\begin{equation} \label{eq-gen_psi}
        \psi(q) = \frac{1}{2} e(q)^TK(q)e(q)
\end{equation}
where $e(q)\in\mathbb{R}^{n_e}$ is the constraint error and $K(q)\in\mathbb{R}^{n_e\times n_e}$ is the symmetric positive definite gain matrix, and $n_s$ is the constraint dimension. 
This form \eqref{eq-gen_psi} is very versatile, as it can represent almost all types of constraint including simple spring, co-rotational finite element model \cite{moita96nme} and even hyper-elastic material with generalized compliance model \cite{macklin19tog}. 
Then we can write as follow using outer product approximation similar to \cite{tournier15tog,macklin19tog,lee21icra}:
\begin{align} \label{hess_psi}
    \begin{split}
        &\frac{\partial \psi}{\partial q} \approx J_e(q)^T K(q) e(q) \\
        &\frac{\partial^2\psi}{\partial q^2} \approx J_e(q)^T K(q)J_e(q) + E(q)
    \end{split}
\end{align}
where $J_e(q)\in\mathbb{R}^{n_e\times n}$ is the constraint Jacobian (i.e., $\frac{\partial e}{\partial q}$) and $E(q)\in\mathbb{R}^{n\times n}$ is the symmetric positive definite damping matrix.
o maintain the exact energy relation along these approximations (or at least, stability preserved), determining $E(q)$ is another meaningful subject, as it is necessary to find an appropriate energy dissipation to preserve the passivity of the system.
Yet in this paper, we do not delve deep into this issue and apply the following two simple policies: 1) user-defined constant damping matrix or 2) symmetric positive definite projection of geometric stiffness matrix \cite{tournier15tog} i.e.,
\begin{align*}
    E(q) = \text{proj}_{\mathbb{S}_n^+}\left[ \frac{\partial J_e}{\partial q}K(q)e(q) \right]
\end{align*}
where $\text{proj}_{\mathbb{S}_n^+}$ is the projection to $n\times n$ symmetric positive definite matrix manifold \cite{you21neuro}.
One of classical ways to compute $\text{proj}_{\mathbb{S}_n^+}$ uses singular value decomposition \cite{nicholas88spd}, however, to reduce the computation time for this, we construct a diagonal matrix where each element is the sum of the absolute values of the elements in each column of the original matrix.

One of the widely used potentials that cannot be represented by \eqref{eq-gen_psi} is gravity potential, however, its Hessian can be ignored by dropping out the derivative of the $J_e(q)$ as above. 
Finally, substituting \eqref{hess_psi} to \eqref{disc_dyn}, dynamics can be represented as in form of \eqref{eq-lindyn} with
\begin{equation} \label{eq-akbk}
\begin{aligned}
    &A_k = \hat{M}_k+ \frac{1}{2}(J_e(q_k)^T K(q_k)J_e(q_k) + E(q_k))t_k \\
    &b_k = \hat{M}_k v_k - C(q_k,v_k)v_k - J_e(q_k)^T K(q_k)e(q_k) + f_{ext,k} 
\end{aligned}
\end{equation}
where $\hat{M}_k=2t_k^{-1}M_k$.
The structure of \eqref{eq-akbk} is very similar to the one used in \cite{tournier15tog,kim17ijrr}.
Although it is based on the linearization of non-linear potential action, constraints are considered in implicit manner, therefore applicable to complex multibody systems including rigid and deformable bodies.
Also here, note that $A_k$ is always a symmetric positive definite matrix.

The form of \eqref{eq-lindyn} is also applicable to the general integration method.
For example, coordinate transform using $M(q)^{\frac{1}{2}}$ introduced in \cite{kim17ijrr} can be utilized to improve the passivity property of the dynamics, as it maintains the linearized form with a symmetric positive definite property of $A_k$. 
For the cases of non-linear integration, that takes constraints into account in a completely implicit manner without linearization, we can write the equation in the form of
\begin{align*} 
    g(\hat{v}_k) = 0
\end{align*}
where $g:\mathbb{R}^n\rightarrow\mathbb{R}^n$ is the non-linear function. 
For instance, for fully-implicit Euler integration \cite{bouaziz14tog}, $g(\hat{v}_k)$ is
\begin{align*} 
M_k\frac{v_{k+1}-v_k}{t_k} + C_k v_k + \nabla\psi(\hat{q}_k) - f_{ext,k} - J_{c,k}^T\lambda_{c,k}
\end{align*}
with $\hat{v}_k=v_{k+1},~\hat{q}_k=\text{update}(q_k,\hat{v}_k,t_k)$. 
Similarly, implicit midpoint integration \cite{volino05cav}, variational integration \cite{murphey09tro} can be represented.
Then solving this equation using the Newton method is equivalent to
\begin{align} \label{eq-newton} 
    \hat{v}_{k} \leftarrow \hat{v}_{k} - \frac{\partial g}{\partial\hat{v}_{k}}g(\hat{v}_{k})
\end{align}
and \eqref{eq-newton} can be simplified as
\begin{align*} 
    A_{k}\hat{v}_{k} = b_{k} + J_{c,k}^T\lambda_{c,k} 
\end{align*}
which is same as \eqref{eq-lindyn}.
Here, $A_{k}$ is a symmetric positive definite matrix, with some quasi-Newton style approximation.
Meanwhile, still there remains room for expansion to be directly applicable to the non-linear integration methods, especially those that can be expressed in an optimization form \cite{bouaziz14tog}.
However, we will remain the part as future work.

For specific cases, algorithms such as \cite{feather00icra} and \cite{wang19tog} can be utilized for effective factorization or linear solving of $A_k$. 
They are not considered here as they do not apply to general multibody dynamics.

\subsection{Signorini-Coulomb Condition} \label{subsec-ncp}

\begin{figure}[t] 
    \centering
    \subfigure[Open]{
    \includegraphics[width=2.5cm]{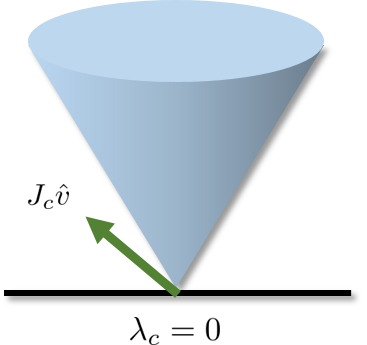}
    }
    \subfigure[Stick]{
    \includegraphics[width=2.5cm]{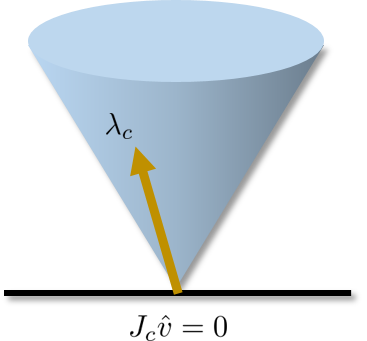}
    }
    \subfigure[Slip]{
    \includegraphics[width=2.5cm]{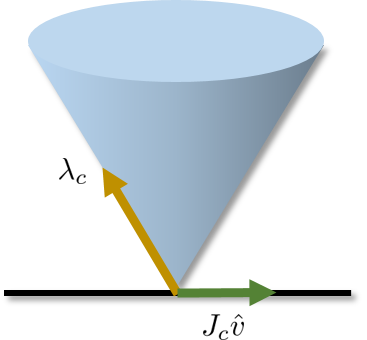}
    }
    \caption{Three cases induced from SCC. Blue shape represents friction cone set, green line denotes contact frame velocity and yellow line denotes contact impulse.}
    \label{fig-complementarity}
\end{figure}

To reflect the physical properties of the contact robustly and precisely, we consider solving the constructed dynamics \eqref{eq-lindyn} under the following complementarity-based Signorini-Coulomb condition (SCC):
\begin{equation} \label{eq-scc}
    \begin{aligned}
    & 0 \le \lambda_{n,m} \perp J_{n,m}\hat{v} + \phi_{n,m} \ge 0 \\
    & 0 \le \delta_m \perp \mu_m\lambda_{n,m} - \| \lambda_{t,m} \| \ge 0 \\
    &\delta_m \lambda_{t,m} +\mu_m\lambda_{n,m}J_{t,m}\hat{v} = 0 
    \end{aligned}
\end{equation}
for all contact indices $m=\left\{ 1, \cdots ,n_{c} \right\}$ where $\mu_m$ is the friction coefficient, $\phi_{n,m}\in\mathbb{R}^{1}$ is the additional term for penetration compensation and restitution coefficient, and $\delta_m\in\mathbb{R}^{1}$ is the auxiliary variable.
The first line of \eqref{eq-scc} is known as the Signorini condition which prevents penetration in a collision.
The rest parts correspond to the Coulomb friction condition that ensures the contact impulse contained within the friction cone set is defined as
\begin{equation} \label{eq-genfc}
    \mathcal{C}_{\lambda} = \left\{\lambda_n,\lambda_t~\vert~ \lambda_n\ge 0,~\mu\lambda_{n} - \| \lambda_{t} \| \ge 0 \right\}
\end{equation}
and tangential impulse operates in the opposite direction of motion if sliding occurs.
There are three possible behavior outcomes as the result of this condition, which are illustrated in Fig.~\ref{fig-complementarity} - open ($\lambda_{n,m}=0$), stick ($\lambda_{n,m}>0, \delta_m=0$), and slip ($\lambda_{n,m}>0, \delta_m>0$).

\section{Contact Nodalization and Diagonalization} \label{Sec-ics}

The main idea of COND is to solve the original contact problem (i.e., solve \eqref{eq-lindyn} with \eqref{eq-scc}) using the repetitions of surrogate contact problem (i.e., solve surrogate dynamics with \eqref{eq-scc}).
In this section, we will describe how the surrogate dynamics contact problems are constructed and solved, with the concept of contact nodalization and diagonalization.

\subsection{Contact Nodalization} \label{subsec-nodal}

We start with the following definition that categorizes contacts into two types:
\begin{definition}
For any point of contact, it is an S-contact, if it is a part of a collision between a dynamic object and a static environment. Otherwise, it is a D-contact, which is a part of a collision between two dynamic objects. 
\end{definition}

\subsubsection{Nodal contact assumption}

\begin{figure}[t] 
    \centering
    \subfigure[Nodal contact]{
    \includegraphics[width=4cm]{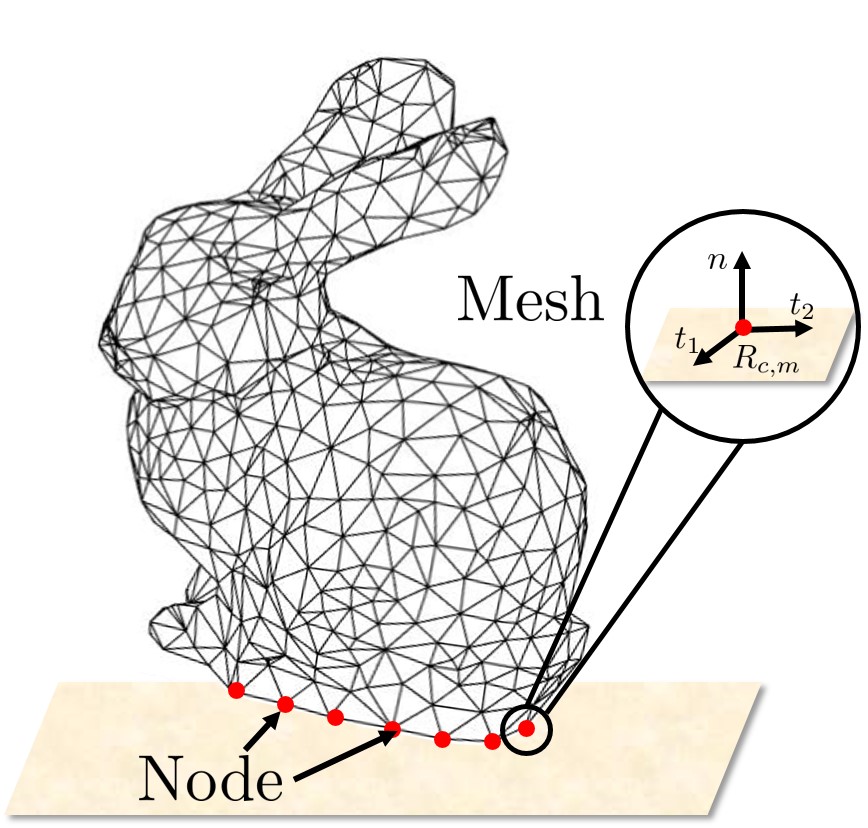}
    \label{fig-nodal}
    }
    \subfigure[Nodal contact with virtual node]{
    \includegraphics[width=4cm]{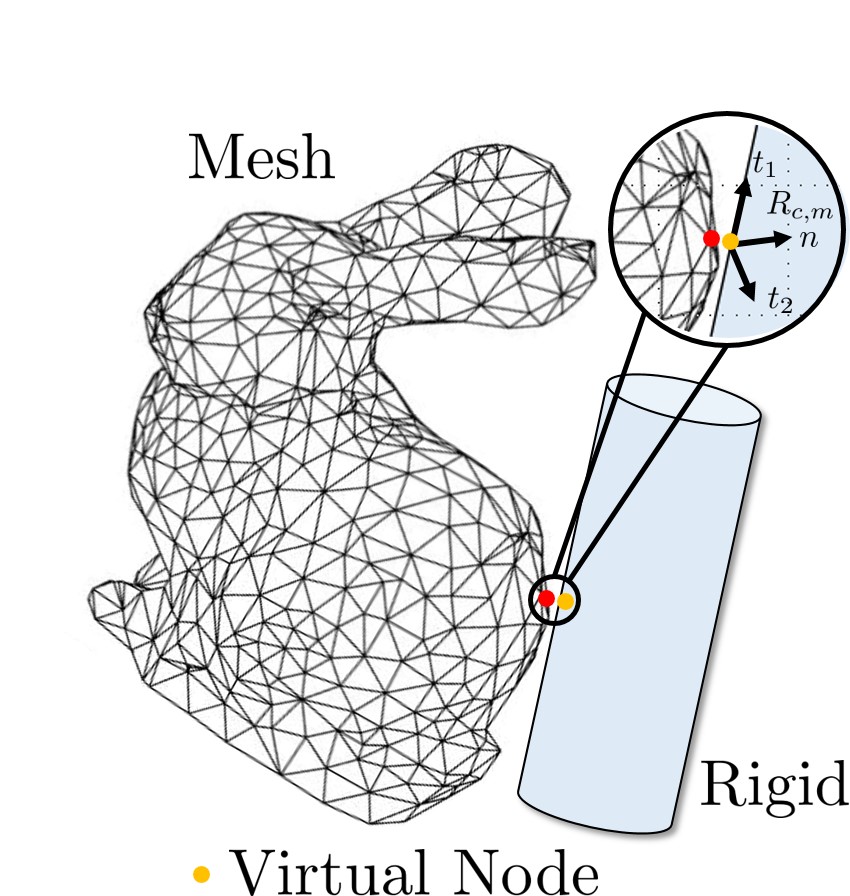}
    \label{fig-virtualnode}
    }
    \caption{Schematic of nodal contact situation. Red points are the nodes that originally exists in the mesh. Orange point is a virtual node that temporally generated on rigid body.}
\end{figure}

As illustrated in Fig.~\ref{fig-nodal}, nodal contact assumes that the points at which the contact acts are only on the nodes that comprise the part of the system coordinates.
That is, under the nodal contact assumption, S-contact is a static environment-node contact, while D-contact is a node-node contact.
Based on this, we can take the important observation that the contact Jacobian $J_c$ in the nodal situation can be represented by the stack of SO(3):
\begin{equation} \label{eq-jcnode}
    \begin{aligned}
    \begin{cases} 
    J_{c,m} = [0,\cdots,R_{c,m},\cdots,0] & \mbox{if S-contact} \\
    J_{c,m} = [0,\cdots,R_{c,m},\cdots,-R_{c,m},\cdots,0] & \mbox{if D-contact}
    \end{cases} 
    \end{aligned}
\end{equation}
where $R_{c,m}\in\mathbb{R}^{3\times 3}$ is the SO(3) matrix for $m$-th contact that converts global frame to nodal contact frame.
Although this concept is reasonable for fine mesh-based systems in practice and successively applied to thin nodal objects such as cloth and hair \cite{li18tog,daviet20tog}, it has not been applied to general multibody systems.
This is mainly because contact points cannot be regarded as a system node, in a coordinate representation such as rigid bodies or joint angles.
However since we are aiming for robotic applications, this extension is essential, and thus we propose the following new technique.

\subsubsection{Virtual node}

To embrace the broader contact situation under the nodal assumption, we make the concept of the virtual node as shown in Fig.~\ref{fig-virtualnode}.  
The mass-less virtual node is temporarily formed when contact occurs if the point is not a predefined node and the contact is treated as occurring on the virtual node.
By this, all the contacts that exist in the system can be fairly considered nodal contacts.
However, for the virtual node concept to be valid, the motion of the virtual node and the point where the contact originally occurred must match.
To implement this, we utilize viscous damping force between the virtual nodes and collision points:
\begin{equation} \label{eq-vforce}
    \begin{aligned}
    &f_v = -k_v(J_v\hat{v}-\hat{v}_v)
    \end{aligned}
\end{equation}
where $k_v$ can be interpreted as a gain, $J_v\in\mathbb{R}^{3n_v \times n}$ is the Jacobian matrix that maps the velocity of the original coordinate to collision point velocity, and $\hat{v}_v\in\mathbb{R}^{3n_v}$ is the representative velocity of virtual nodes while $n_v$ is the number of virtual nodes. 
Note that damping force is sufficient to match the motion since the virtual nodes are created at the exact locations (i.e., no constraint error) at every time step and disappear at the next time step.
Virtual nodes can be generated on any part of the system (e.g., surface on rigid body, interior points between mesh nodes) by simply creating a mapping from the original coordinate to the location of the node. 

Based on \eqref{eq-vforce}, we can reformulate the original dynamics as
\begin{equation} \label{eq-vnode}
    \begin{aligned}
        &\underbrace{\begin{bmatrix}
        A_o+k_vJ_v^TJ_v & -k_vJ_v^T \\
        -k_vJ_v & k_vI
        \end{bmatrix}}_{=A}
        \underbrace{
        \begin{bmatrix}
        \hat{v}_o \\ \hat{v}_v
        \end{bmatrix}}_{=\hat{v}} =
        \underbrace{\begin{bmatrix}
        b_o \\ 0
        \end{bmatrix}}_{=b} + 
        \underbrace{
        \begin{bmatrix}
        J^T_{co} & 0 \\ 0 & J^T_{cv} 
        \end{bmatrix}}_{=J^T_c}
        \underbrace{
        \begin{bmatrix}
            \lambda_{co} \\ \lambda_{cv}
        \end{bmatrix}}_{\lambda_c}
    \end{aligned}
\end{equation}
where $A_o$, $b_o$ are from original dynamics, $\lambda_{co}$, $\lambda_{cv}$, $J_{co}$, and $J_{cv}$ are respectively the contact impulse on original/virtual nodes and contact Jacobian on original/virtual nodes. 

\subsubsection{Analysis on virtual node}
It is easy to see that \eqref{eq-vnode} is structurally identical to \eqref{eq-lindyn}.
We also demonstrated in Prop. 1 that the symmetric positive definite property of dynamic matrix $A$ is preserved even after nodalization.

\begin{proposition} 
$A$ in \eqref{eq-vnode} is a symmetric positive definite matrix. 
\end{proposition}
\begin{proof}
Since $A_o$ is a symmetric positive definite matrix, $A$ is at least symmetric and positive semi-definite. Now suppose that $x=[x_1;x_2]$ exists that satisfies $x^TAx=0$. Then $x_1$ must be zero vector to make $x_1^TA_ox_1 = 0$. Then $x^TAx=k_v x_2^Tx_2=0$ holds, which show $x_2$ is also a zero vector and it denotes $A$ is positive definite.
\end{proof}

Now let us identify how the viscous damping force \eqref{eq-vforce} affects the system dynamics. 
Dynamic equation \eqref{eq-vnode} can be rewritten as
\begin{align} 
    &A_o\hat{v}_o = b_o - k_vJ_v^TJ_v\hat{v}_o+k_vJ_v^T\hat{v}_v + J_{co}^T\lambda_{co} \label{eq-vnode1} \\
    &J_{cv}^T\lambda_{cv} = -k_vJ_v\hat{v}_o + k_v\hat{v}_v  \label{eq-vnode2}
\end{align}
Substituting \eqref{eq-vnode2} to \eqref{eq-vnode1}, we can obtain
\begin{align} \label{eq-vnode3}
    A_o\hat{v}_o = b_o + J_v^TJ_{cv}^T\lambda_{cv} + J_{co}^T\lambda_{co}
\end{align}
which implies that the formulation is consistent with the original dynamics structure, without any additional force.
However, the contact solution from \eqref{eq-vnode} satisfies the SCC with respect to virtual node velocity rather than the original contact point velocity i.e.,
\begin{align*}
    J_{cv}\hat{v}_v = J_{cv}(J_v\hat{v}_o + k_v^{-1}J_{cv}^T\lambda_{cv})
\end{align*}
instead of $J_{cv}J_v\hat{v}_o$.
This means the term $k_v^{-1}J_{cv}J_{cv}^T\lambda_{cv}$ induces the ``perturbation'' on the contact constraint. 
However, we theoretically show that as $k_v$ increases, this perturbation term vanishes in Prop. 2, and thus constraint drift (e.g., penetration) can be certifiably avoided if $k_v$ is large enough.
Note that, this is clearly different from spring-damper based force injection, as they creates additional force in the system and does not guarantee the satisfaction of the constraint.
For empirical evaluation, see Sec.~\ref{Sec-result}.

\begin{proposition} 
The contact formulation using virtual nodes converges to the original contact formulation as $k_v\rightarrow \infty$.
\end{proposition}
\begin{proof} 
From the dissipative property of SCC (normal complementarity, opposite friction direction),
\begin{align*}
    \lambda_{cv}^TJ_{cv}(J_vA_o^{-1}J_v^T+k_v^{-1}I)J_{cv}^T\lambda_{cv} + \lambda_{cv}^TJ_{cv}J_vA_o^{-1}b'_o \le 0 
\end{align*}
holds where $b'_o=b_o+J_{co}^T\lambda_{co}$.
Suppose that if $k_v^{-1}J_{cv}^T\lambda_{cv}$ does not converges to $0$ as $k_v\rightarrow\infty$, which means $\| J_{cv}^T\lambda_{cv} \|$ goes to $+\infty$. 
Now notice that the left-hand side of the inequality is a summation of $k_v\| k_v^{-1}J_{cv}^T\lambda_{cv} \|^2$ and the quadratic term with respect to $J_v^T J_{cv}^T\lambda_{cv}$.
Then it is clear to see that the former goes to $+\infty$, and the latter also has a lower bound as $A_o^{-1}$ is a symmetric positive definite matrix.
This means that $k_v^{-1}J_{cv}^T\lambda_{cv}$ must converge to $0$ to satisfy the inequality.
Therefore as $k_v\rightarrow\infty$, virtual node-based contact formulation converges to original contact formulation as the perturbation becomes zero and the relation between $J_{cv}J_v\hat{v}$ and $\lambda_{cv}$ is exactly leveraged.
\end{proof}

It is also important to note that in \eqref{eq-vforce}, utilizing implicit representation $\hat{v},\hat{v}_v$ rather than $v,v_v$ is critical, as it leads \eqref{eq-vnode3} does not introduce any additional force.
As a result, even with sporadic generations of multiple virtual nodes and high $k_v$, the formulation can maintain stability.

Our virtual node-based formulation is similar to a concept of slack variable, so may have a disadvantage that the dimension of the system increases as the number of virtual nodes increases.
Nonetheless, the nodal transformation significantly contributes to ability of our algorithm to reduce the burden of the $\mathcal{O}(n^3)$ complexity in matrix operations\cite{todorov12iros}, as will be discussed further below.

\subsection{Derivation of Surrogate Dynamics} \label{subsec-subp}

Let us first consider solving \eqref{eq-vnode} with only Signorini condition in \eqref{eq-scc}.
Then we can find that it is equivalent to Karush–Kuhn–Tucker (KKT) conditions of the following velocity-level optimization problem:
\begin{equation} \label{eq-vopt}
    \begin{aligned}
        \min_{\hat{v}}&~\frac{1}{2}\hat{v}^TA\hat{v}-b^T\hat{v} \\
        \text{s.t.}&~\hat{v} \in \mathcal{C_V}
    \end{aligned}
\end{equation}
where $\mathcal{C_V}=\{\hat{v} \mid J_{n}\hat{v} + \phi_n \ge 0 \}$ is the feasible set of velocity with $\phi_n=[\phi_{n,1},\cdots,\phi_{n,m}]^T$. Here and hereafter, notations for $k$-th time step are omitted for simplification, however all the components are still time-varying.
This means that contact problems without friction can be replaced with solving \eqref{eq-vopt}. One way to solve \eqref{eq-vopt} is the projected gradient descent method \cite{boyd14fto} which takes the following steps:
\begin{align}
    &\hat{v}^* \leftarrow \hat{v}^l-W(A\hat{v}^l-b) \label{eq-pgd1} \\
    &\hat{v}^{l+1} \leftarrow \Pi^W_{\mathcal{C_V}}(\hat{v}^*) \label{eq-pgd2}
\end{align}
where $l$ is the iteration loop index, $W\in\mathbb{R}^{n\times n}$ is the symmetric positive definite step size matrix for gradient descent.
Here, $W$ can alternatively interpreted as an inverse of \textit{surrogate} dynamics matrix, which will be a key component of our derivation.
Also, $\Pi(\cdot)$ is defined as
\begin{equation} \notag
    \begin{aligned}
        &\Pi^W_{\mathcal{C_V}}(x) = \argmin_{y\in \mathcal{C_V}} \|y-x \|_W^2
    \end{aligned}
\end{equation}
which is the projection from $x$ on the convex set $\mathcal{C_V}$, with respect to weighted Euclidean distance. 
Since $A$ is positive definite and $\mathcal{C_V}$ is convex, the projected gradient descent method will well converge to the solution of \eqref{eq-vopt} if the problem is feasible.
However, since \eqref{eq-scc} also include the impulse constraint like friction direction constraint, it is not enough to deal with the generic contact problem. 
Instead, we modify the projection step \eqref{eq-pgd2} as follow:
\begin{equation} \label{eq-vd}
    \begin{aligned}
        &W^{-1} \hat{v}^{l+1} \leftarrow W^{-1}\hat{v}^* + J_{c}^T\lambda_{c}
    \end{aligned}
\end{equation}
Noticing the similarity between \eqref{eq-vd} and \eqref{eq-lindyn}, \eqref{eq-vd} can be considered as a surrogate dynamics, which is constructed from $W$ and $\hat{v}^*$. 
Then the update step can be interpreted as solving the contact problem with respect to this surrogate dynamics.
Now recalling the form of \eqref{eq-delassus}, $\lambda_c$ is equivalent to
\begin{equation} \label{eq-subp}
    \begin{aligned}
    \lambda_c = \text{SOL}(\eta_c,\Gamma_c)
    \end{aligned}
\end{equation}
where $\Gamma_c = J_cWJ_c^T \in \mathbb{R}^{3n_c\times 3n_c}$ can be interpreted as a surrogate Delassus operator, $\eta_c = J_c\hat{v}^*\in \mathbb{R}^{3n_c}$, and SOL denotes the solution that $\lambda_c$ and $J_c\hat{v}^{l+1}=\Gamma_c\lambda_c+\eta_c$ satisfies SCC.
After solving this contact problem \eqref{eq-subp}, update step is performed as
\begin{align} \label{eq-vup}
    &\hat{v}^{l+1} \leftarrow \hat{v}^* + W J_{c}^T\lambda_{c} 
\end{align}
Unlike the projection step in projected gradient descent that only enforces velocity-level constraint, this new velocity fixed point iteration (V-FPI, i.e., iteration of \eqref{eq-pgd1}, \eqref{eq-subp} and \eqref{eq-vup}) directly achieves the contact condition \eqref{eq-scc}. 

\begin{remark} [Consistency]
If the fixed-point iteration converges i.e.,
\begin{align*}
    &\hat{v}^{l+1} =\hat{v}^l-W(A\hat{v}^l-b) + W J_{c}^T\lambda_{c} =\hat{v}^l
\end{align*}
then the solution exactly consistent with the original dynamics
\begin{align*}
    A\hat{v}^l = b + J_{c}^T\lambda_{c}
\end{align*}
as $W$ is supposed to be a non-singular matrix.
Therefore, we can find that surrogate dynamics \eqref{eq-vd} eventually accurately reflects the dynamics condition if the iteration converges, see also Sec.~\ref{Sec-result}.
For the analysis of the convergence, see Sec.~\ref{Sec-conv}.
\end{remark}

\subsection{Contact Diagonalization} \label{subsec-contactdiag}

In our V-FPI process above, $\hat{v}^*$ computing step \eqref{eq-pgd1} and $\hat{v}^{l+1}$ updating step \eqref{eq-vup} are simple processes, yet the part that solves surrogate contact problem \eqref{eq-subp} may be complicated.
However, we find that selecting an appropriate surrogate dynamics matrix (i.e., $W$) under contact nodalization can make our surrogate Delassus operator $\Gamma_c = J_{c}WJ_{c}^T$ to be a diagonal matrix for all situations and the process can be drastically simplified. 
The proposition below demonstrates how it works.

\begin{proposition} 
Under contact nodalization, suppose 
\begin{align*}
    &W = \text{diag}(w_{11},\cdots,w_{nn})
\end{align*}
with $w_{i_1i_1}=w_{i_2i_2}=w_{i_3i_3}=\bar{w}_i \in \mathbb{R}^1$ for all $i$-th node that is in contact, where $i_1,i_2,i_3$ are indices corresponding to the node. Then the surrogate Delassus operator $\Gamma_c$ can be written as follow:
\begin{align} \label{eq-diagonal}
\begin{split}
    \Gamma_c&=\text{diag}(\Gamma_{c,1},\cdots,\Gamma_{c,{n_c}}) \\ &=\text{diag}(\gamma_{c,1} I_{3\times 3},\cdots,\gamma_{c,{n_c}} I_{3\times 3}) \\
    \gamma_{c,m} &=
    \begin{cases}
    \bar{w}_i, & \mbox{if $i$-th node is in S-contact} \\       
    \bar{w}_i+\bar{w}_j, & \mbox{if $i,j$-th node are in D-contact}
    \end{cases}
\end{split}
\end{align}
where $I$ denotes the identity matrix.
\end{proposition}
\begin{proof}
Recalling the structure of \eqref{eq-jcnode}, for S-contact,
\begin{align*}
    \Gamma_{c,m} = R_{c,m} (\bar{w}_i I_{3\times 3}) R_{c,m}^T = \bar{w}_i R_{c,m}R_{c,m}^T = \bar{w}_i I_{3\times 3}
\end{align*}
and for D-contact,
\begin{align*}
    \Gamma_{c,m} 
    &= R_{c,m} (\bar{w}_i I_{3\times 3}) R_{c,m}^T + R_{c,m} (\bar{w}_j I_{3\times 3}) R_{c,m}^T \\
    &= \bar{w}_i R_{c,m}R_{c,m}^T + \bar{w}_j R_{c,m}R_{c,m}^T \\
    &= (\bar{w}_i+\bar{w}_j) I_{3\times 3}
\end{align*}
holds from the definition of SO(3).
\end{proof}

Prop. 3 demonstrates that under a certain structure of $W$, extracting the components of $W$ is sufficient to construct the surrogate Delassus operator, without any matrix-matrix multiplication or factorization.
Therefore, we can obtain significant advantages in time and memory compared to the original Delassus operator (i.e. $J_cA^{-1}J_c^T$) assembly.
Also, the simple structure of the Delassus operator is significantly advantageous for the contact solving process (see Sec.~\ref{subsec-csolver}).

With this contact ``diagonalization'' method, we have built the basic structure of our simulation algorithm - COND, which is summarized in Alg. 1.

\begin{algorithm} [t]
\caption{COND}
\label{alg1}
\begin{algorithmic}[1] 
\While{Simulation loop}
\State{Construct original dynamics}
\State{Perform collision detection with contact nodalization}
\State{Construct dynamics $A,b,J_c$ based on \eqref{eq-vnode}}
\State{Initialize $l=1, \hat{v}^1$}
\While{V-FPI loop}
\State{Determine step size matrix $W$}
\State Compute $\hat{v}^*$ using \eqref{eq-pgd1}
\State{Construct $\Gamma_c$ based on diagonalized property \eqref{eq-diagonal}}
\State $\text{Solve } \eqref{eq-subp} \text{ using Alg. 2}$
\State Update $\hat{v}^{l+1}$ using \eqref{eq-vup}
\State Compute residual $\theta = \| \hat{v}^{l+1}-\hat{v}^l\|$
\If{$\theta < \theta_{th}$ or $l=l_{max}$} 
\State \textbf{break}
\EndIf
\State $l\leftarrow l+1$
\EndWhile
\State{Update state using $\hat{v}^{l+1}$}
\EndWhile
\end{algorithmic}
\end{algorithm}

\subsection{Solving Surrogate Dynamics Problem} \label{subsec-csolver}

Now the remaining part is how to solve \eqref{eq-subp}.
In typical solvers, although the contact conditions \eqref{eq-scc} are constraints that are independent with each contact, coupling between each contact still exists since the Delassus operator is dense. 
These coupling terms make the global iteration and relaxation process essential for contact solvers when dealing with multi-contact situations \cite{hwangbo18ral,horak19ral}.
However, in COND, the diagonalized property of $\Gamma_c$ established in Sec.~\ref{subsec-contactdiag} resolves this problem. 
Convenience induced from the property of $\Gamma_c$ in the contact solving process can be summarized as follows:
\begin{itemize}
    \item Coupling relaxation through global iteration (e.g., Gauss-Seidel) is unnecessary since each contact situation is completely decoupled. Furthermore, completely parallel computation is possible for each contact.
    \item Even for each single contact problem, coupling between normal and tangential impulse does not exist and it leads the exact solution to be obtained more simply without the need of numerical methods such as bisection \cite{hwangbo18ral}, Newton \cite{daviet11tog}, etc.
\end{itemize}
Overall, our contact solver consists only of a single linear solving of $-\Gamma_c^{-1}$ (which is very simple since $\Gamma_c$ is a diagonal matrix) with parallelized local projection step on $\mathcal{C}_{\lambda}$. We summarize this contact solver in Alg. 2.

\begin{figure}[t] 
    \centering
    \subfigure[Strict]{
    \includegraphics[width=2.5cm]{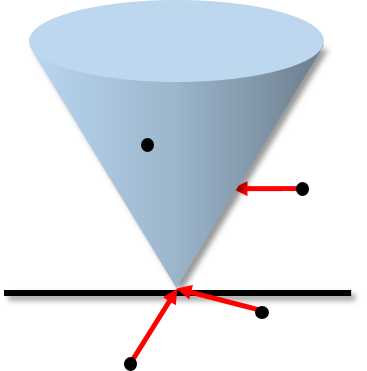}
    \label{strict}
    }
    \subfigure[Proximal]{
    \includegraphics[width=2.5cm]{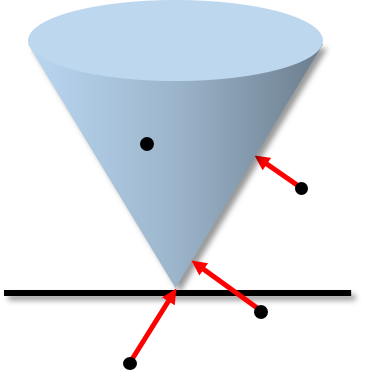}
    \label{proximal}
    }
    \caption{Two ProjectFC schemes: strict and proximal. Black dot: value before projection, Red arrow: projection.}
    \label{projectfc}
\end{figure}

\begin{algorithm}[t]
\caption{One-shot/parallelized contact solver}
\label{alg2}
\begin{algorithmic}[1]
\For{$m = 1 \text{ to } n_c$} \textbf{in parallel} 
\State $\lambda_{c,m}^{*} = -(\gamma_{c,m} I_{3\times 3})^{-1}(\eta_{c,m}+\phi_{c,m})$
\State $\lambda_{c,m} = \text{ProjectFC}(\lambda_{c,m}^{*})$
\EndFor
\end{algorithmic}
\end{algorithm}

The friction cone projection step (denoted as ProjectFC in Alg. 2) is necessary to enforce the output of the contact solver to satisfy the SCC.
If temporary value (i.e., $\lambda_{c,m}^{*}$ in Alg. 2) is inside $\mathcal{C}_{\lambda}$, ProjectFC simply yields the same value as input. Otherwise, $\lambda_{c,m}^{*}$ is projected on the surface of $\mathcal{C}_{\lambda}$. In this paper, we use two projection schemes - ``strict'' and ``proximal'' operator as illustrated in Fig~\ref{projectfc}. 
Each can be written as
\begin{equation} \label{eq-operator}
    \begin{aligned}
        \text{Strict: } &\lambda_{n,m} = \max(\lambda_{n,m}^{*},0) \\
        &\lambda_{t,m} = \Pi^I_{\mathcal{C}_{\lambda_{n,m}}}(\lambda_{t,m}^{*}) \\
        \text{Proximal: } &\lambda_{c,m} = \Pi^I_{\mathcal{C}_{\lambda}}(\lambda_{c,m}^{*})
    \end{aligned}
\end{equation}
where $\mathcal{C}_{\lambda_{n,m}}$ is the cross-section of $\mathcal{C}_{\lambda}$ cut vertically from $\lambda_{n,m}$.
It can be easily verified that for both strict operator and proximal operator, we can obtain the solution in a very simple analytic form. Also note that $\phi_c$ is only related to normal component, therefore $\phi_{c,m}=[\phi_{n,m};0;0]$.

We figure out that the two projection schemes have the following trade-offs in our solver: 1) the result of the strict operator exactly satisfies the SCC condition \eqref{eq-scc}, yet sufficient condition for convergence is not always feasible\footnote{From the fundamental limitation of SCC. In practice, we do not experience the convergence problem. For detailed discussions and results, see Sec.~\ref{Sec-conv},\ref{Sec-result}.} and 2) proximal operator can always guarantee the convergence of the solver, yet the solution does not exactly satisfy \eqref{eq-scc}.
Also, the result from the proximal operator is equivalent to the convex contact model proposed in \cite{tasora10coa,todorov14icra}, therefore having a unique solution under strong convexity.

\subsubsection{Strict operator}

The strict operator can exactly achieve SCC in one-shot which is shown in Prop. 4.

\begin{proposition}
Output of Alg. 2 with the strict operator in \eqref{eq-operator} is the unique solution of \eqref{eq-subp}.
\end{proposition}
\begin{proof}
From $\Gamma_{c,m}=\gamma_{c,m} I_{3\times 3}$, normal component is completely decoupled from tangential component as 
\begin{align*}
    &J_{n,m}\hat{v} = \gamma_{c,m}\lambda_{n,m} + \eta_{n,m}    
\end{align*}
where $\eta_{c,m} = [\eta_{n,m},\eta_{t_1,m},\eta_{t_2,m}]^T$.
If $\lambda_{n,m}^{*}>0$, $\lambda_{n,m}=\lambda_{n,m}^{*}$ is the only solution that $J_{n,m}\hat{v}+\phi_{n,m}=0$ is satisfied. 
Else, $\lambda_{n,m}=0$ is the only solution as $J_{n,m}\hat{v}+\phi_{n,m}>0$ since $\gamma_{n,m}>0$.
Therefore, normal components are uniquely determined, and satisfy the complementarity condition is satisfied.
For tangential components, stick case is trivial. 
For slip case,
\begin{equation} \label{eq-gbeta}
    \begin{aligned}
    &\gamma_{c,m}\lambda_{t_1,m} + \eta_{t_1,m} = -\beta \lambda_{t_1,m} \\
    &\gamma_{c,m}\lambda_{t_2,m} + \eta_{t_2,m} = -\beta \lambda_{t_2,m}
    \end{aligned}
\end{equation}
must be satisfied for $\beta>0$. Substituting \eqref{eq-gbeta} to boundary of $\mathcal{C}_{\lambda}$, $\beta$ is uniquely determined as
\begin{align*}
    \beta = -\gamma_{c,m} + \frac{\| \lambda_{t,m}\|}{\mu_m\lambda_{n,m}}
\end{align*}
and therefore $\lambda_{t,m}$ is uniquely determined and equivalent to the result of the strict operator. 
\end{proof}

\subsubsection{Proximal operator}
 
As depicted in Fig.~\ref{projectfc}, a result of the proximal operator is different from the strict operator, which implies that the proximal operator-based formulation contains some approximation in SCC, as shown in Prop. 5.

\begin{proposition}
Output of Alg. 2 with the proximal operator in \eqref{eq-operator} is the unique solution of \eqref{eq-subp}, while SCC is replaced with following relaxed form:
\begin{align} \label{eq-relaxedscc}
\begin{split}
    & 0 \le \lambda_{n,m} \perp J_{n,m}\hat{v} + \phi_{n,m} - \mu_m \delta_m \ge 0 \\
    & 0 \le \delta_m \perp \mu_m\lambda_{n,m} - \| \lambda_{t,m} \| \ge 0 \\
    &\delta_m \lambda_{t,m} + \mu_m\lambda_{n,m}J_{t,m}\hat{v} = 0 
\end{split}
\end{align}
\end{proposition}
\begin{proof}
Proximal operator solves following optimization problem: 
\begin{equation} \label{eq-prox-opt}
\begin{aligned}
    \lambda_{c,m} = \min_{\lambda} \frac{1}{2}\lambda^T\Gamma_{c,m}\lambda + \lambda^T(\eta_{c,m}+\phi_{c,m}) \quad \text{s.t.} \quad \lambda \in \mathcal{C}_{\lambda}
\end{aligned}
\end{equation}
for all contact index $m=\left\{ 1, \cdots ,n_{c} \right\}$. 
Then KKT conditions of the problem can be written as
\begin{align} 
        &J_{n,m}\hat{v} + \phi_{n,m} - \vartheta' - \mu_m\vartheta = 0 \label{eq-proxkkt1}\\
    &J_{t,m}\hat{v} + \vartheta(\lambda_{t,m}/\|\lambda_{t,m}\|) = 0 \label{eq-proxkkt2}\\
    &\vartheta'\lambda_{n,m} = 0\label{eq-proxkkt4} \\
    &\vartheta(-\mu_m\lambda_{n,m} + \|\lambda_{t,m}\|) = 0 \label{eq-proxkkt5}\\
    &\vartheta,\vartheta' \ge 0 \label{eq-proxkkt6} \\
    &\lambda_{c,m} \in \mathcal{C}_{\lambda} \label{eq-proxkkt7}
\end{align}
for $m=\left\{ 1, \cdots ,n_{c} \right\}$ where $\vartheta,\vartheta'$ is the Lagrange multipliers. From \eqref{eq-proxkkt1}, \eqref{eq-proxkkt4}, and \eqref{eq-proxkkt6}, we can find that
\begin{align*}
    &J_{n,m}\hat{v} + \phi_{n,m} - \mu_m\vartheta = \vartheta' \ge 0 \\
    &\lambda_{n,m}(J_{n,m}\hat{v} + \phi_{n,m} - \mu_m\vartheta) = \vartheta'\lambda_{n,m} = 0
\end{align*}
holds and it is equivalent to relaxed normal complementarity condition in the statement.
For tangential component condition, open case is trivial from \eqref{eq-proxkkt7}.
Also from \eqref{eq-proxkkt5}, $\vartheta = 0$ holds for stick case and $\| J_{t,m}\hat{v} \| = 0$ is satisfied. Finally, \eqref{eq-proxkkt2} is equivalent to the condition for slip case.
Also, since \eqref{eq-prox-opt} is the strictly convex optimization problem, the solution is unique.
\end{proof}

If V-FPI iteration with proximal operator converges, original dynamics \eqref{eq-lindyn} is satisfied and
\begin{align*}
    \Gamma_c\lambda_c + \eta_c &= \Gamma_c\lambda_c + J_c(\hat{v}-WJ_c^T\lambda_c) \\
    &=A_c\lambda_c + b_c
\end{align*}
holds where $A_c = J_cA^{-1}J_c^T$ and $b_c = J_cA^{-1}b$.
Thus, $\lambda_c = \text{SOL}(\eta_c,\Gamma_c) = \text{SOL}(b_c,A_c)$ and we can find that converged solution of COND using proximal operator is equivalent to the solution of CCP \cite{tasora10coa} which can be written as 
\begin{equation} \label{eq-ccp}
\begin{aligned}
    \lambda_c = &\min_{\lambda=[\lambda_1;\cdots;\lambda_{n_c}]} \frac{1}{2}\lambda^TA_c\lambda + \lambda^T(b_c+\phi_c) \\
    &\text{s.t.} \quad \lambda_m \in \mathcal{C}_{\lambda} \quad m=\left\{ 1, \cdots ,n_{c} \right\} 
\end{aligned}
\end{equation}
Compare the conditions in Prop. 5 with \eqref{eq-scc}, we can find that Signorini condition is relaxed. Consequently, as also shown in \cite{horak19ral}, the solution of CCP may generate unplausible dynamic behavior (e.g., gliding effect during sliding).
However, as mentioned in \cite{tasora10coa}, the moderate use of $\phi$ can reduce the effect of the approximation of normal conditions, and in practice it is employed in well-known simulators such as MuJoCo and Chrono as it works quite robustly.

\begin{remark}[Uniqueness]
Although Prop. 4 and 5 conclude the uniqueness of the solution to the surrogate problem, they do not directly represent the uniqueness of the solution to the original problem.
The uniqueness of the solution is related to the number of fixed-points in the V-FPI.
\end{remark}

\subsection{Extensions} \label{subsec-invt}

Contact diagonalization allows COND to easily extend with various contact models.
Here we present some examples.
\subsubsection{Invertible contact}
In \cite{todorov14icra}, an invertible contact model based on the regularization term is proposed as follows:
\begin{align} \label{eq-invt}
\begin{split}
    \lambda_c = \text{SOL}(b_c,A_c+\Omega_c)
\end{split}
\end{align}
where $\Omega_c\in\mathbb{R}^{3n_c\times 3n_c}$ is the symmetric positive definite regularization term.
Under the model, $\lambda_c$ is uniquely determined by solving the following problem:
\begin{align*}
    \lambda_c = \text{SOL}(J_c\hat{v},\Omega_c)
\end{align*}
Invertible contact model allows for reversely calculating contact impulse (i.e., $\lambda_c$) from the velocity result (i.e., $\hat{v}$)), which can be useful in fields such as contact-implicit trajectory optimization.
Such a scheme can straightforwardly be included in COND, as it can be easily verified that \eqref{eq-invt} is also solvable using our framework by simply using $\Gamma_c+\Omega_c$ instead of $\Gamma_c$ in \eqref{eq-subp}.

\subsubsection{Anisotropic Friction}

\begin{figure}[t] 
    \centering
    \includegraphics[width=7.5cm]{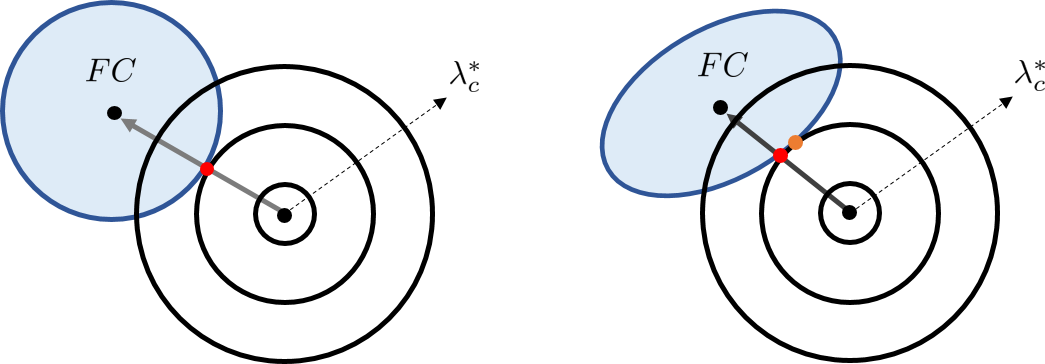}
    \caption{Illustration of the strict operator projection result (orange) for the isotropic friction model (left) and the anisotropic friction model (right). The solution with the friction direction opposite to the sliding direction (red) differs in anisotropic friction.}
    \label{mdp_fc}
\end{figure}

So far, we focus on isotropic friction model, which use the same friction coefficient for all tangential directions. 
However in some cases, anisotropic friction modeling is necessary \cite{erleben20cgf} for accurate surface modeling. 
For example, ellipsoidal cone as
\begin{align*} 
    \mathcal{C}_{\lambda} = \left\{\lambda_n,\lambda_t~\vert~ \lambda_n\ge 0,~\lambda_{n}^2 \ge \frac{\lambda_{t_1}^2}{\mu_1^2} + \frac{\lambda_{t_2}^2}{\mu_2^2} \right\}
\end{align*}
can be utilized instead of isotropic cone.
In this paper, we use the maximal dissipation principle (MDP \cite{preclik17cm}) to model the behavior under anisotropic friction.
The validity of the MDP formulation has been shown in \cite{yu16iros,ma18ral}.
Here, due to our diagonalization process, MDP can be conveniently accommodated, using the strict operator.
Details are shown in Prop. 6.

\begin{proposition}
Output of Alg. 2 with the strict operator is the maximal dissipation solution.
\end{proposition}
\begin{proof}
Maximal dissipation principle can be written as
\begin{align*}
    \lambda_{c} = &\min_{\lambda=[\lambda_1;\cdots;\lambda_{n_c}]} \frac{1}{2}\lambda^T\Gamma_{c}\lambda + \lambda^T(\eta_{c}+\phi_c) \\ 
    \text{s.t.}~&0 \le \lambda_{n,m} \perp \gamma_{c,m}\lambda_{n,m} +\eta_{n,m} + \phi_{n,m} \ge 0 \\
    &\lambda_m \in \mathcal{C}_{\lambda} \quad m=\left\{ 1, \cdots ,n_{c}\right\}
\end{align*}
We can easily find that the cost function in MDP is proportional to the square of the distance from $\lambda_{t,m}^{*}$ as in Fig.~\ref{mdp_fc} while the normal component is already determined to satisfy the Signorini condition. Therefore, the optimal solution is equivalent to the minimum distance projection result. 
\end{proof}

Note that, in the case of anisotropic friction, the MDP is not equivalent to that the friction is in the opposite direction which is also depicted in Fig.~\ref{mdp_fc}.
Also, the projection result of the strict operator may not be represented analytically, thus numerical methods (e.g., bisection) may be required.
However, as Alg. 2 supports parallelization for all contacts, the problem can be handled without bottlenecks.

\subsection{Complexity} \label{subsec-complexity}

Summarizing the preceding contents, it is established that all components of V-FPI in COND (Alg. 1) are made up of simple scalar algebraic operations and matrix-vector multiplication. 
More precisely, the first multiplication of $W$ in \eqref{eq-vd} and \eqref{eq-vup} has $\mathcal{O}(n)$ complexity for both time and space since $W$ is a diagonal matrix.
One-shot contact solving process (Alg. 2) also possesses $\mathcal{O}(n)$ complexity since the construction of $\Gamma_c$ does not involve any multiplication, but only element extraction from $W$, while contact solving is interpreted as $n_c$ (usually proportional to $n$) number of simple parallelizable operations.
Also, since $J_c$ can be treated as just a stack of SO(3) matrices, matrix-vector multiplication on $J_c$ has also $\mathcal{O}(n)$ complexity.
Therefore, the only part that has over linear complexity is a computation of $A\hat{v}^l$. Matrix-vector multiplication requires $\mathcal{O}(n^2)$ complexity if $A$ is dense, but in many cases, especially for deformable body parts, $A$ contains many $0$ components. As a result, COND shows the complexity near $\mathcal{O}(n)$ - see Sec.~\ref{Sec-result}. Note that despite $A$ is sparse, $A^{-1}$ is generally fully dense, therefore computation of Delassus operator $J_cA^{-1}J_c^T$ is still has a complexity near $O(n^3)$. 

\subsection{Chebyshev Acceleration}
We find that Chebyshev acceleration \cite{wang15tog} can be utilized as an efficient plug-in to accelerate the V-FPI in COND. 
For iterative linear solver to solve $A\hat{v}=b$ defined as
\begin{equation} \notag
    \hat{v}^{l+1}= A_1^{-1}(A_2\hat{v}^l-b)
\end{equation}
where $A=A_1-A_2$, Chebyshev acceleration method proposes iteration scheme as
\begin{align} \label{eq-chebyup}
    \hat{v}^{l+1}= \nu_{l+1}(A_1^{-1}(A_2\hat{v}^l-b)-\hat{v}^{l-1})+\hat{v}^{l-1}
\end{align}
with
\begin{equation} \notag
    \nu_{l+1} =\frac{2P_{l}(\frac{1}{\varrho})}{\varrho P_{l+1}(\frac{1}{\varrho})}
\end{equation}
where $P$ denotes the Chebyshev polynomial and $\varrho$ is the spectral radius of $A_1^{-1}A_2$. 
Looking closely at Alg. 1, we can find that our solver is quite similar to the linear solving scheme with $A_1=-W^{-1}, A_2=W^{-1}-A$. Therefore, it can be expected that applying Chebyshev acceleration to COND will be effective. The algorithm is clarified in Alg. 3.

\begin{algorithm}[t]
\caption{V-FPI with Chebyshev acceleration}
\label{alg3}
\begin{algorithmic}[1] 
\While{V-FPI loop}
\State{Determine step size matrix $W$}
\State {Compute $\hat{v}^*$ using \eqref{eq-pgd1}} 
\State{Construct $\Gamma_c$ based on diagonalized property \eqref{eq-diagonal}}
\State $\text{Solve } \eqref{eq-subp} \text{ using Alg. 2}$
\State Update $\hat{v}^{**}$ instead of $\hat{v}^{l+1}$ using \eqref{eq-vup}
\State $\nu_{l+1} = 
\begin{cases}
1 & \mbox{if }l<l_s \\
\frac{2}{2-\varrho^2}& \mbox{if }l=l_s \\
\frac{4}{4-\varrho^2\nu_l}& \mbox{if }l>l_s \\
\end{cases}$
\State Update $\hat{v}^{l+1}$ using \eqref{eq-chebyup}
\State Compute residual $\theta = \| \hat{v}^{l+1}-\hat{v}^l\|$
\If{$\theta < \theta_{th}$ or $l=l_{max}$} 
\State \textbf{break}
\EndIf 
\State Compute $\varrho$ using \eqref{eq-varrho}
\State $l\leftarrow l+1$
\EndWhile
\end{algorithmic}
\end{algorithm}
Here, $l_s$ is the index for slowly starting acceleration as in \cite{wang15tog} which is helpful to avoid oscillation at the beginning of the iteration. Also, since spectral radius $\varrho$ is equivalent to the largest of the absolute values of the eigenvalues, its time complexity becomes an issue for a large size matrix. For this reason, we use a simple approximation of it instead as
\begin{align} \label{eq-varrho}
    \varrho = \min(\frac{\| \hat{v}^l-\hat{v}^{l-1} \|}{\| \hat{v}^{l-1}-\hat{v}^{l-2} \|},1)
\end{align}
Intuitively, this formulation reflects the degree of a contraction property, which is directly related to the spectral radius. The min operator is necessary because monotonic decreasing of residual is not always guaranteed with the acceleration schemes. 
In practice, when Chebyshev acceleration is applied, the residual of V-FPI reduces much faster, yet with some chattering. To reduce the resulting instability, we utilize the under-relaxation strategy (i.e., $0<u<1$ in Alg. 3) which is proposed in \cite{wang15tog}.

\section{Convergence Analysis} \label{Sec-conv}
In this section, sufficient conditions for convergence of V-FPI in COND will be discussed concretely.
We first utilize the Banach Fixed-Point Theorem \cite{richard07fpt} as a basis to discuss convergence i.e.,

\begin{lemma}[Banach Fixed-Point Theorem \cite{richard07fpt}] \label{banach} 
If $f:X\rightarrow X$ is a contraction mapping on non-empty complete metric space, fixed-point iteration $x\leftarrow f(x)$ always converges to a unique fixed-point.
\end{lemma}

The iteration process of COND can be written in the following function form:
\begin{align*}
f(\hat{v}) &= (I-WA)\hat{v} + Wb + W J_c^T\lambda_c \\
           &= \hat{v}^* + W J_c^T\lambda_c
\end{align*}
Here $\lambda_c$ is the function of $\hat{v}^*$, yet cannot be expressed analytically in usual cases. However, our one-shot contact solver described in Sec.~\ref{subsec-csolver} facilitates access to this difficulty. 

\begin{definition} 
V-FPI in COND is contractible if the contraction property of the iteration can be guaranteed under the proper matrix $W$.
\end{definition}
By Lemma 1, if we show the contractibility of V-FPI, it implies that we can ensure the convergence of iteration and the uniqueness of the problem solution.
We find that exact and proximal operators exhibit different mathematical properties so the analysis for each will be dealt with separately.
Note that we use 2-norm as a distance metric so further notation of $\| \cdot \|$ denotes 2-norm of a matrix/vector.
Also, $\sigma(\cdot)$ will refer eigenvalues of the matrix, and $\bar{v}_i=[\hat{v}_{i_1};\hat{v}_{i_2};\hat{v}_{i_3}]\in\mathbb{R}^3$ will indicate representative velocity of $i$-th node.

\begin{remark}[Nodalization]
The contact acts only on the (virtual) nodes among the total system degrees of freedom. For this reason, the value of the elements of $f(\hat{v})$ other than the indices of the node to which the contact acts is equal to that of $\hat{v}^*$ - the description of the corresponding content will be omitted in the proofs.
\end{remark}

\subsection{Strict Operator}

Since the strict operator gives the solution that exactly satisfy SCC, it means that if our V-FPI always converges, there is always a unique solution to the original contact problem. 
However, feasibility and uniqueness of the contact NCP solution cannot be always guaranteed \cite{preclik17cm,mazhar15tog}.
For this reason, dealing with complete convergence property of the strict operator is a fundamentally hard problem, unlike the proximal operator case (which will be explained later).
Therefore, in this paper, we consider specific type of system.

\begin{definition} 
A system is a completely-nodal system if the system is composed only of particle nodes, and contact only acts on the nodes.
\end{definition}

Even if a definition of the completely-nodal system does not cover all commonly used system dynamics expressions, it can ``represent'' any system by adopting a sufficiently large number of nodes.
We find that for the completely-nodal system with a sufficiently small time step, the contractibility of V-FPI can be guaranteed. For this, we first present the following lemma.

\begin{lemma}
Strict operator is Lipschitz continuous.
\end{lemma} 
\begin{proof}
From the property that projection on convex set is contraction \cite{studer08thesis} (and therefore, Lipschitz continuous), we can find that max operator is Lipschitz continuous, as $\lambda_n\ge0$ is the invariant convex set. Therefore, we only need to show $\Pi_{\mathcal{C}_{\lambda_{n,m}}}$ is Lipschitz continous. Suppose two inputs of $\Pi_{\mathcal{C}_{\lambda_{n,m}}}$ are $\lambda_1^*=[\lambda_{1,n}^*;\lambda_{1,t_1}^*;\lambda_{1,t_2}^*]$ and $\lambda_2^*=[\lambda_{2,n}^*;\lambda_{2,t_1}^*;\lambda_{2,t_2}^*]$ with outputs are $\lambda_1$ and $\lambda_2$ correspond to each. 
Now consider following triangular inequality:
\begin{align*}
    \| \lambda_1- \lambda_2 \| 
    &\le \| \lambda_1 - \lambda_3 \| + \| \lambda_2 - \lambda_3 \|
\end{align*}
with $\lambda_3^*=[\lambda_{2,n}^*;\lambda_{1,t_1}^*;\lambda_{1,t_2}^*]$ and output $\lambda_3$.
Then we can prove each of the following two:
\begin{itemize}
    \item $\| \lambda_1- \lambda_3 \| \le \zeta \| \lambda_1^*- \lambda_3^* \|$ for finite constant $\zeta$ 
\end{itemize}
If $\lambda_1^*,\lambda_3^*\in\mathcal{C}_{\lambda}$, it is trivial.
If $\lambda_1^*,\lambda_3^*\notin\mathcal{C}_{\lambda}$, it holds from the fact that following derivative is bounded:
\begin{align*}
{\left\lVert \frac{\partial\lambda_{t}}{\partial\lambda_n^*} \right\rVert}
= {\left\lVert \frac{\mu\lambda_{t}^*}{\| \lambda_{t}^* \|} \right\rVert}
= \mu
\end{align*}
If $\lambda_1^*\in\mathcal{C}_{\lambda},\lambda_3^*\notin\mathcal{C}_{\lambda}$, it can be derived from following inequality:
\begin{align*}
    \| \lambda_1- \lambda_3 \| \le \| \acute{\lambda}_1- \acute{\lambda}_3 \| \le \zeta\| \acute{\lambda}^*_1- \acute{\lambda}^*_3 \| = \zeta \| \lambda_1^*- \lambda_3^* \|
\end{align*}
where $\acute{\lambda}$ is the output of $\Pi_{\mathcal{C}_{\lambda_{n,m}}}$ from input $\acute{\lambda}^*=[\lambda_n^*;d\lambda_{t_1}^*;d\lambda_{t_2}^*]$ with large enough $d$ that satisfies $\acute{\lambda}^*\notin\mathcal{C}_{\lambda}$.

\begin{itemize}
    \item $\| \lambda_2- \lambda_3 \| \le \zeta \| \lambda_2^*- \lambda_3^* \|$ for finite constant $\zeta$
\end{itemize}
Since $\lambda_2$ and $\lambda_3$ share same cross-section $\mathcal{C}_{\lambda_{n,m}}$, the property is directly derived from the contraction property of convex set projection.

Finally, following inequalities hold
\begin{align*}
    &\| \lambda_1^* - \lambda_3^* \|\le\| \lambda_1^* - \lambda_2^* \| \quad \| \lambda_2^* - \lambda_3^* \|\le\| \lambda_1^* - \lambda_2^* \|
\end{align*}
and therfore Lipschitz continuity is established.
\end{proof}

Here, the Lipschitz constant is dependent on friction coefficients and increases as the friction coefficient gets bigger. Based on the lemma, we can derive the following theorem.

\begin{theorem}
Suppose that the system is a completely-nodal system. Then V-FPI in Alg. 1 with the strict operator is contractible for sufficiently small time step $t$. 
\end{theorem}
\begin{proof}
Consider that $m$-th contact is S-contact on $i$-th node. From
\begin{align*}
    \bar{v}_{i}^{1,l+1} - \bar{v}_{i}^{2,l+1} =\bar{v}_{i}^{1,*} - \bar{v}_{i}^{2,*} + \bar{w}_{i}R_{c,m}^T(\lambda_{c,m}^1-\lambda_{c,m}^2)
\end{align*}
with triangular inequality,
\begin{align*} 
    \| \bar{v}_{i}^{1,l+1} - \bar{v}_{i}^{2,l+1} \| \le \| \bar{v}_{i}^{1,*} - \bar{v}_{i}^{2,*} \| + \bar{w}_{i} \| \lambda_{c,m}^1-\lambda_{c,m}^2 \|
\end{align*}
holds. From Lemma 2, let us Lipschitz constant of the strict operator as $\zeta_m$. Then
\begin{equation} \notag
    \| \lambda_{c,m}^1-\lambda_{c,m}^2 \| \le \bar{w}_{i}^{-1} \zeta_m \| \bar{v}_i^{1,*}-\bar{v}_i^{2,*} \|
\end{equation}
is satisfied and we can find that
\begin{align} \label{eq-thm1-1}
    \| \bar{v}_{i}^{1,l+1} - \bar{v}_{i}^{2,l+1} \| \le
    (1+\zeta_m) \| \bar{v}_i^{1,*}-\bar{v}_i^{2,*} \| 
\end{align}
holds. Now consider that $m$-th contact is D-contact on $i,j$-th node. Then similarly to above,
\begin{align*}
    &\| \bar{v}_{i}^{1,l+1} - \bar{v}_{i}^{2,l+1} \| \\
    &\le \| \bar{v}_{i}^{1,*} - \bar{v}_{i}^{2,*} \| + 
    \zeta_m (\| \bar{v}_{i}^{1,*} - \bar{v}_{i}^{2,*} \|+\| \bar{v}_{j}^{1,*} - \bar{v}_{j}^{2,*} \|) \\
    &\| \bar{v}_{j}^{1,l+1} - \bar{v}_{j}^{2,l+1} \| \\
    &\le \| \bar{v}_{j}^{1,*} - \bar{v}_{j}^{2,*} \| + 
    \zeta_m (\| \bar{v}_{j}^{1,*} - \bar{v}_{j}^{2,*} \|+\| \bar{v}_{i}^{1,*} - \bar{v}_{i}^{2,*} \|) 
\end{align*}
is satisfied and we can find that
\begin{equation} \label{eq-thm1-2}
    \begin{aligned}
    &\| \bar{v}_{i}^{1,l+1} - \bar{v}_{i}^{2,l+1} \|^2 +
    \| \bar{v}_{j}^{1,l+1} - \bar{v}_{j}^{2,l+1} \|^2 \\
    &\le\zeta'_m (\| \bar{v}_{i}^{1,*} - \bar{v}_{i}^{2,*} \|^2 +
    \| \bar{v}_{j}^{1,*} - \bar{v}_{j}^{2,*} \|^2) 
    \end{aligned}
\end{equation}
holds for finite $\zeta'_m$. Applying \eqref{eq-thm1-1} and \eqref{eq-thm1-2} to all contact,
\begin{align*}
    \| \hat{v}^{1,l+1} - \hat{v}^{2,l+1} \| \le \zeta \| \hat{v}^{1,*} - \hat{v}^{2,*} \|
\end{align*}
is established for finite $\zeta$. Therefore from $\hat{v}^*=(I-WA)\hat{v}^l$, if $\| I-WA \| < \frac{1}{\zeta}$, V-FPI satisfies contraction property. Now recall the structure of $A$ in \eqref{eq-akbk}. Suppose $W=\hat{M}^{-1}$ as it satisfies the structure in Prop. 3 as the system is a complete-nodal system. Finally, from
\begin{align*}
    \| I-WA \| = \frac{t^2}{2}\| M^{-1}(J_e^TKJ_e + E) \| 
\end{align*}
we can reach $\| I-WA \| < \frac{1}{\zeta}$ by lowering $t$.
\end{proof}

Although Thm. 1 does not describe the complete convergence property of the solver, it provides a partial answer to the uniqueness and existence of the contact NCP solution in the multi-contact situation which is previously unknown.
We also believe that this result suggests that in the case of a completely-nodal system, small local deformation of the contact part is possible even when it is a rigid body, so the contradictions arising from assuming an ideal rigid body can be alleviated.
Extensive study for generalization of the completely-nodal system will remain as future work. 
Note that aside from the theoretical analysis, we empirically observe that the solver is robustly convergent in the general case. 

\subsection{Proximal Operator}
As mentioned earlier, when the proximal operator is used, it is the same as the solution of a convex optimization, which indicates the existence and uniqueness of the solution. 
We find that our algorithm can always ensure convergence in this case.
To show this, we first utilize the following lemma.
\begin{lemma} \label{theta}
Consider the following equations:
\begin{align*}
    &\lambda^1_{c,m} = \Pi^I_{\mathcal{C}_\lambda}(-\gamma_{c,m}^{-1}(\eta_{c,m}^1+\phi_{c,m})) \\
    &\lambda^2_{c,m} = \Pi^I_{\mathcal{C}_\lambda}(-\gamma_{c,m}^{-1}(\eta_{c,m}^2+\phi_{c,m}))
\end{align*}
Then for $\xi_{c,m}=-\gamma_{c,m}^{-1}(\eta_{c,m}+\phi_{c,m})-\lambda_{c,m}$, following is holds:
\begin{align*}
    (\xi_{c,m}^1-\xi_{c,m}^2)^T(\lambda_{c,m}^1-\lambda_{c,m}^2) \ge 0  
\end{align*}
\end{lemma}
\begin{proof}
Since $\mathcal{C}_{\lambda}$ is a convex set, proximal operator has following property \cite{studer08thesis}:
\begin{align*}
&x=\Pi^I_{\mathcal{C}_\lambda}(x^*) \rightarrow x^* \in x + \mathcal{N}_{\mathcal{C}_\lambda}(x)
\end{align*}
where $\mathcal{N}$ is the normal cone defined as
\begin{align*}
&\mathcal{N}_{\mathcal{C}_\lambda}(x) = \left\{y\vert y^T(x'-x) \le 0 \right\} \quad \forall x'\in \mathcal{C}_\lambda,x \in \mathcal{C}_\lambda
\end{align*}
Therefore, $\xi_{c,m}^1\in\mathcal{N}_{\mathcal{C}_\lambda}(\lambda_{c,m}^1),\xi_{c,m}^2\in\mathcal{N}_{\mathcal{C}_\lambda}(\lambda_{c,m}^2)$ and
\begin{align*}
    &(\xi_{c,m}^1)^T(\lambda_{c,m}^1-\lambda_{c,m}^2) \ge 0  \\
    &(\xi_{c,m}^2)^T(\lambda_{c,m}^2-\lambda_{c,m}^1) \ge 0
\end{align*}
is satisfied from the definition of normal cone.
\end{proof}

\begin{figure}[t] 
    \centering
    \includegraphics[width=5.5cm]{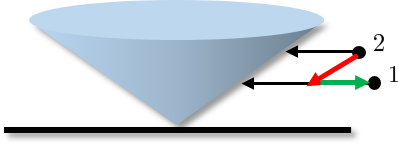}
    \caption{Counter-example of Lemma 3 on strict operator. Red vector denotes $\lambda_c^1-\lambda_c^2$ and green vector denotes $\xi_c^1-\xi_c^2$.}
    \label{counter}
\end{figure}

As depicted in Fig.~\ref{counter}, Lemma 3 does not always hold for the strict operator, since it takes a two-stage projection rather than finding the nearest point.
Based on this lemma, we can prove the contractibility of our algorithm, which is shown in Thm. 2.

\begin{theorem}
V-FPI in Alg. 1 with the proximal operator is contractible.
\end{theorem}
\begin{proof}
Suppose that $m$-th contact is a S-contact on $i$-th node. From the 
\begin{align*}
&\bar{v}_{i}^{1,l+1} - \bar{v}_{i}^{2,l+1} =\bar{v}_{i}^{1,*} - \bar{v}_{i}^{2,*} + \bar{w}_{i} R_{c,m}^T(\lambda_{c,m}^1-\lambda_{c,m}^2)
\end{align*}
we can derive following equations:
\begin{align*}
&\| \bar{v}_{i}^{1,l+1} - \bar{v}_{i}^{2,l+1} \|^2 \\
&= \| \bar{v}_{i}^{1,*} - \bar{v}_{i}^{2,*} \|^2 + \bar{w}_{i}^2\|\lambda_{c,m}^1-\lambda_{c,m}^2 \|^2 \\
&+ 2\bar{w}_{i}(\bar{v}_{i}^{1,*}-\bar{v}_{i}^{2,*})^TR_{c,m}^T(\lambda_{c,m}^1-\lambda_{c,m}^2) \\
&= \| \bar{v}_{i}^{1,*} - \bar{v}_{i}^{2,*} \|^2 - \bar{w}_{i}^2\|\lambda_{c,m}^1-\lambda_{c,m}^2 \|^2 \\
&- 2\bar{w}_{i}(\xi_{c,m}^{1}-\xi_{c,m}^{2})^T(\lambda_{c,m}^1-\lambda_{c,m}^2) 
\end{align*}
Then from Lemma 3, we can find that
\begin{align} \label{eq-thm2-1}
&\| \bar{v}_{i}^{1,l+1} - \bar{v}_{i}^{2,l+1} \|^2 \le \| \bar{v}_{i}^{1,*} - \bar{v}_{i}^{2,*} \|^2 
\end{align}
holds. Now suppose that $m$-th contact is D-contact on $i,j$-th nodes and $\bar{w}_{i}=\bar{w}_{j}$ holds. From
\begin{align*}
  &\bar{v}_{i}^{1,l+1} - \bar{v}_{i}^{2,l+1} = (\bar{v}_{i}^{1,*} - \bar{v}_{i}^{2,*})+\bar{w}_{i} R_{c,m}^T(\lambda_{c,m}^1-\lambda_{c,m}^2) \\
  &\bar{v}_{j}^{1,l+1} - \bar{v}_{j}^{2,l+1} = (\bar{v}_{j}^{1,*} - \bar{v}_{j}^{2,*})-\bar{w}_{i} R_{c,m}^T(\lambda_{c,m}^1-\lambda_{c,m}^2) 
\end{align*}
we can derive following equations:
\begin{align*}
    &\| \bar{v}_{i}^{1,l+1} - \bar{v}_{i}^{2,l+1} \|^2 +
    \| \bar{v}_{j}^{1,l+1} - \bar{v}_{j}^{2,l+1} \|^2 \\
    &= \| \bar{v}_{i}^{1,*} - \bar{v}_{i}^{2,*} \|^2 +
    \| \bar{v}_{j}^{1,*} - \bar{v}_{j}^{2,*} \|^2 \\
    &+ 2\bar{w}_{i} (\lambda_{c,m}^1-\lambda_{c,m}^2)^TR_{c,m}((\bar{v}_{i}^{1,*}-\bar{v}_{j}^{1,*})- (\bar{v}_{i}^{2,*}-\bar{v}_{j}^{2,*})) \\
    &+2\bar{w}_{i}^2 \| \lambda_{c,m}^1-\lambda_{c,m}^2 \|^2 \\
    &=\| \bar{v}_{i}^{1,*} - \bar{v}_{i}^{2,*} \|^2 +
    \| \bar{v}_{j}^{1,*} - \bar{v}_{j}^{2,*} \|^2 \\
    &-4\bar{w}_{i}^2(\lambda_{c,m}^1-\lambda_{c,m}^1)^T(\xi_{c,m}^1-\xi_{c,m}^2) -2\bar{w}_{i}^2 \| \lambda_{c,m}^1-\lambda_{c,m}^2 \|^2
\end{align*}
Then from Lemma 3, we can find that
\begin{equation} \label{eq-thm2-2}
    \begin{aligned}
    &\| \bar{v}_{i}^{1,l+1} - \bar{v}_{i}^{2,l+1} \|^2 +
    \| \bar{v}_{j}^{1,l+1} - \bar{v}_{j}^{2,l+1} \|^2 \\
    &\le\| \bar{v}_{i}^{1,*} - \bar{v}_{i}^{2,*} \|^2 +
    \| \bar{v}_{j}^{1,*} - \bar{v}_{j}^{2,*} \|^2 
    \end{aligned}
\end{equation}
Applying \eqref{eq-thm2-1} and \eqref{eq-thm2-2} to all contacts,
\begin{align*}
\| \hat{v}^{1,l+1} - \hat{v}^{2,l+1} \| \le \| \hat{v}^{1,*} - \hat{v}^{2,*} \|
\end{align*}
is established. Therefore, from $\hat{v}^*=(I-WA)\hat{v}^l$,
contraction property is satisfied if $\| I-WA \| < 1$ holds.
Now consider $W = \alpha I$ with $0<\alpha < 2\sigma^{-1}_{max}(A)$. 
It satisfies the structure in Prop. 3 and $\bar{w}_{i}=\bar{w}_{j}$ for D-contact that we suppose.
Also, since $A$ is a symmetric matrix, $I-WA=I-\alpha A$ is also a symmetric matrix. Then
\begin{align*}
    \| I-WA \| = \rho(I-WA)
\end{align*}
holds. Also from the positive definite property of $A$ shown in Prop. 1, we can easily find that
\begin{align*}
    \rho(I-WA) < 1
\end{align*}
can be ensured, therefore, the V-FPI is contractible.
\end{proof}

The above implies an interesting insight; contacts are contributing property for the contraction property of the iteration process. 
For this reason, we observe that COND converges faster than proceeding linear solving on a gradient basis in the absence of contact, and convergence is still well established even in the case that $\| I-WA \|$ is slightly over $1$.

\subsection{Determination of $W$} \label{subsec-SSM}
As discussed above, choosing $W$ is the important part of the solver since the sufficient condition to guarantee the global convergence is directly related to the value of $\|I-WA\|$. As already known,
\begin{align*}
& W = \alpha I \text{ where } \alpha = \frac{2}{\sigma_{min}(A)+\sigma_{max}(A)}
\end{align*}
can enforce $\| I-WA \| < 1$ with the conditions required in Prop. 3.
However, computing the eigenvalues of a large size matrix is generally time-consuming and using the fixed step size for every iteration, and index is not favorable in terms of performance. 
Here, we propose some efficient and robust strategies to determine $W$.
These strategies  work well for all the cases we tested.

\subsubsection{Frobenius norm minimization}
Frobenius norm of $I-WA$ which can be written as:
\begin{align*}
&\| I-WA \|_F^2 =  n-\sum_{i=1}^nw_{ii}a_{ii}+\sum_{i=1}^n\| A_{i*} \|^2
\end{align*}
where subscript $i*$ denotes $i$-th row of matrix.
Considering $w_{i_1i_1}=w_{i_2i_2}=w_{i_3i_3}=\bar{w}_{i}$ for each node that in contact, first-order necessary condition to minimize $\| I-WA \|_F^2$ can be written as
\begin{align*}
&\bar{w}_{i} = \frac{a_{i_1i_1}+a_{i_2i_2}+a_{i_3i_3}}{\| A_{i_1*} \|^2+\| A_{i_2*} \|^2+\| A_{i_3*} \|^2} 
\end{align*}
If the nodes are in D-contact, as described above, $6$ components need to be the same value and the result can be derived similarly to above. For the index corresponds to the object that is not in contact, as it does not affect $\Gamma_c$,
\begin{align*}
&w_{ii} = \frac{a_{ii}}{\| A_{i*} \|^2}
\end{align*}
can be used as in \cite{tarazaga09cma}. 

\subsubsection{Barzilai-Borwein step size}

Barzilai-Borwein method \cite{dai05nm} is the popular method to solve large-scale optimization problems.
Even though the problem we are dealing with is not an optimization problem, we find that the strategy can also be well adopted to our framework. Here, $\alpha$ can be determined by two ways:
\begin{align*}
    \alpha_{bb1}^l = \frac{s_l^Ts_l}{s_l^Tz_l} \quad \alpha_{bb2}^l=\frac{s_l^Tz_l}{z_l^Tz_l}
\end{align*}
where $s_l = \hat{v}^l-\hat{v}^{l-1}$ and $z_l=A(\hat{v}^l-\hat{v}^{l-1})$. 
Any of the two can be selected or can be used alternately as in \cite{dai05nm}.

We observe that the two methods show similar performance levels and that better methods vary depending on the scenario.
Frobenius norm minimization takes longer to construct $W$ since it corresponds to $\mathcal{O}(n^2)$ complexity. We find that recycling $W$ for several time step sections can circumvent this, utilizing the fact that the value of $W$ does not significantly change for the adjacent time step. 
Note that reuse of $W$ does not affect the solver accuracy, since the variables associated with the true physical state (e.g., $A,b,J_c$) remain unchanged.

\section{Results and Evaluation} \label{Sec-result}

In this section, simulation results of diverse scenarios using COND with performance evaluation will be presented. 

\subsection{Implementation Details}

\subsubsection{Tools}
For our implementation, we use Intel Core i5-7500 CPU 3.40GHz (Quad-Core), OpenGL as rendering tool, C++ Eigen as matrix computation library, and C++ OpenMP as parallelization library.

\subsubsection{Matrix format}
To handle large-size matrices efficiently, we use compressed sparse column (CSC) format to store $A$ and $K$. $J_c$ is stored as a stack of $3\times 3$ matrices. The matrices which are always guaranteed to be diagonal like $\Gamma_c$ and $W$ are stored in vector format.

\subsubsection{Collision detection}
Collision detection is performed per time step.
Here we use self-developed code, which is mainly based on vertex-volume detection.

\subsubsection{Solver specification}
For all examples, we use fixed time step as $ t_k=10~\rm{ms}$.
The strict operator is used in contact solver for all results except in Sec.~\ref{subsec-invertible}.

\subsubsection{Warm start}
A warm start is utilized to promote the performance of the solver. 
For the part that is not a virtual node, we directly set $\hat{v}^1$ as the $\hat{v}$ value from the previous time step. The warm start of virtual nodes velocity is calculated by multiplying Jacobian mapping $J_v$ to the original state (i.e., state without virtual nodes) value.
Note that this velocity warm start is not available in I-FPI based methods.
Instead, we utilize the warm start in I-FPI to the previous time step value for the contact index overlapping the previous step.

\subsection{Effect of Virtual Nodes} \label{subsec-effvn}

\begin{figure}[t] 
    \centering
    \subfigure[Box sliding. Analytic solution (green box), COND (red box) and Penalty2 (blue box). Left: start, Right: after sliding.]{
    \includegraphics[width=4cm]{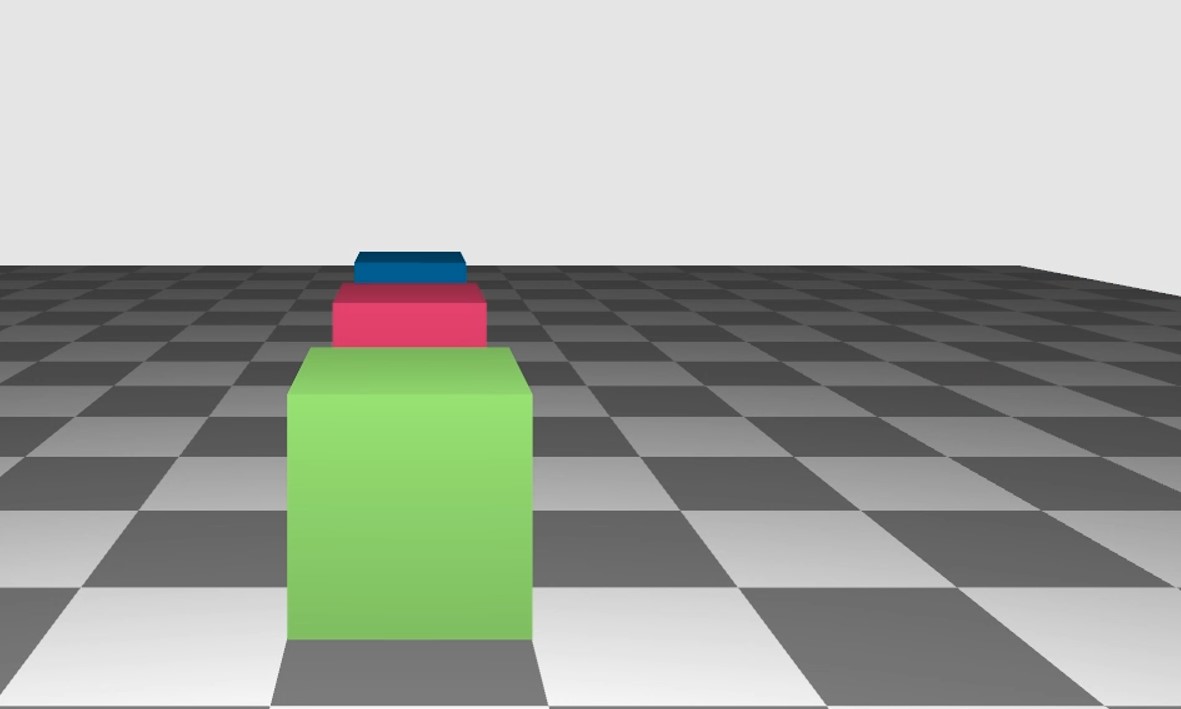}
    \includegraphics[width=4cm]{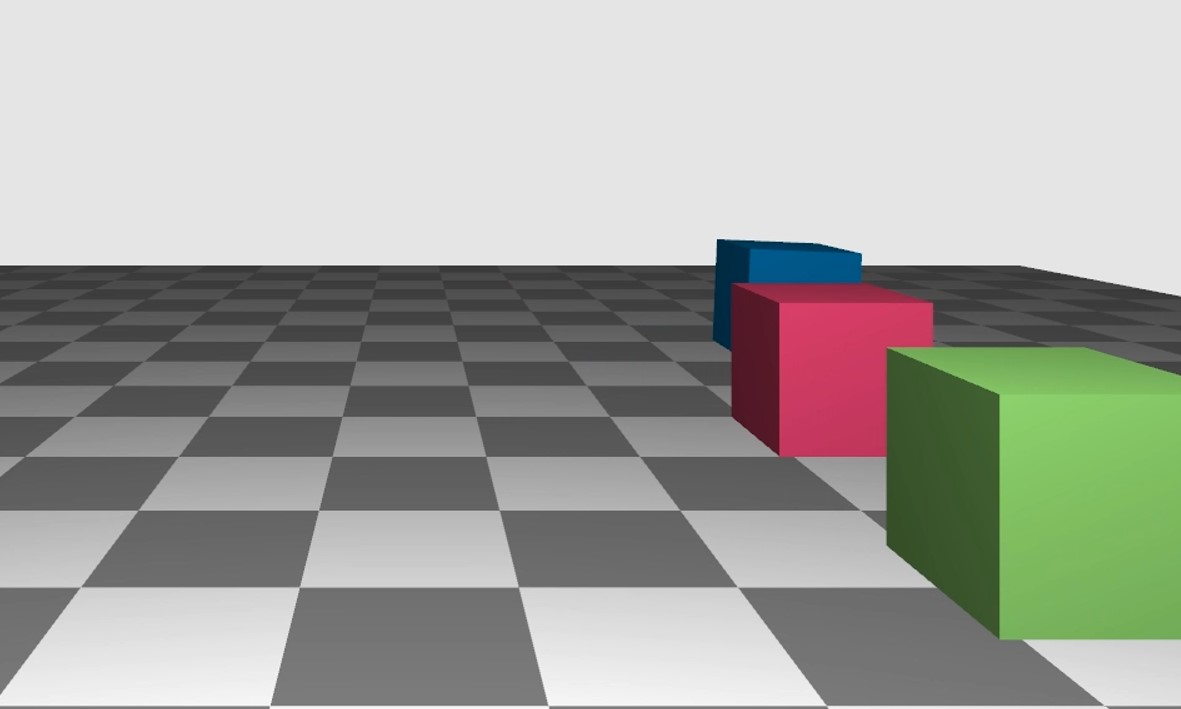}
    }
    \subfigure[Position of box]{
    \includegraphics[width=7cm]{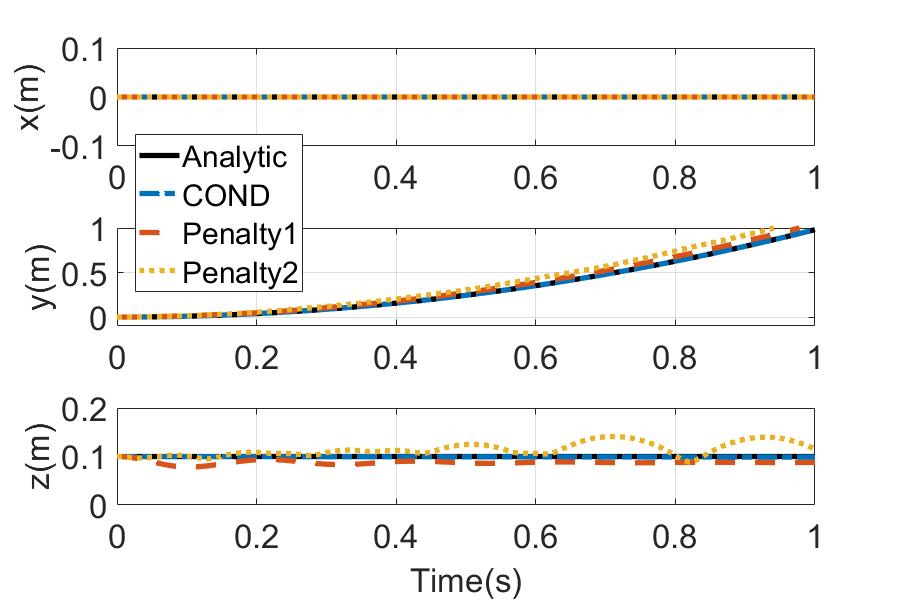}
    }
    \caption{Snapshots and result comparison plot for box sliding simulation. COND: virtual nodes gain $k_v=10^5$.}
    \label{fig-boxresult}
\end{figure}

\begin{table}[t]
\small
\centering
\renewcommand{\arraystretch}{1.5}{
\resizebox{8.4cm}{!}{
\begin{tabular}{|c|c|c|c|}
\hline
Formulation & $x$ error [mm] & $y$ error [mm] & $z$ error [mm] \\
\hline
\hline
COND ($k_v=10^3$) & 0.0000 & 0.1226 & 0.6131 \\
\hline 
COND ($k_v=10^4$) & 0.0000 & 0.0122 & 0.0612 \\
\hline 
COND ($k_v=10^5$) & 0.0000 & 0.0012 & 0.0061 \\
\hline 
Penalty1 & 0.0000 & 30.23 & 12.24 \\
\hline 
Penalty2 & 0.0000 & 63.87 & 15.32 \\
\hline 
\end{tabular}
}
}
\caption{Position error comparison results from analytic solutions of various contact formulations (virtual node and penalty). Averaged values over time steps.} 
\label{table-boxresult}
\end{table}

In this subsection, we intend to characterize our formulation based on nodalization.
Without virtual nodes, our framework does not take any relaxation. 
With virtual nodes, we show that the solution can precisely satisfy the SCC \eqref{eq-scc} for a reasonably large $k_v$ in Prop. 2.
To demonstrate the property, we configure the simple test cases that apply a constant force ($y$-direction) to the cube-shaped rigid box placed on the floor.
The parameters are set to length $0.2~\rm{m}$, mass $0.5~\rm{kg}$, the friction coefficient is set to $0.2$, and the four vertices of the bottom face are set as contact points.
We compare the result with an analytic solution, and the solution from the spring-damper-based penalty contact model \cite{yamane06icra}.
Two gain sets are selected for the penalty model (Penalty1 and Penalty2), and the gains in Penalty2 are all gains in Penalty1 multiplied by 10.

Comparison results are depicted in Fig.~\ref{fig-boxresult} and Table~\ref{table-boxresult}.
Virtual node-based formulation utilized in COND shows a small error, and its magnitude decreases as $k_v$ increases.
This demonstrates the consistency of virtual node formulation with original dynamics, which is described in Prop. 2.
On the other hand, the spring-damper penalty model shows a much larger error.
Also, the result varies greatly depending on the gain parameters and implausible behavior (e.g., bouncing motion) is generated in Penalty2.
This also indicates that large gain values do not guarantee a reduction in error, unlike virtual node formulation.

\subsection{Multibody Simulation Examples} \label{subsec-multibodyexp}

In this subsection, several multibody simulation examples are implemented using COND and evaluated.

\subsubsection{Baselines}
As previously mentioned, due to their versatility, I-FPI-based methods are widely used in the robotics community, as well as in open-source simulation software (MuJoCo, Bullet, Chrono, etc.).
Here, we implement the following two I-FPI baseline algorithms and characterize the difference between I-FPI and V-FPI:
1) projected Gauss-Seidel (PGS\cite{horak19ral}); and
2) accelerated projected gradient descent (APGD\cite{mazhar15tog}).
Since they rely on Delassus operator assembly, factorization of the $A$ matrix is required.
We use the SimplicialLLT function of the Eigen Sparse library for factorization, which is a state-of-the-art open-source implementation.
Residual in I-FPI is converted to velocity space using the map $A^{-1}J^T$ and evaluated to be consistent with V-FPI.

\subsubsection{Performance index}

Perform comparisons are conducted based on the following indices: 
\begin{itemize}
    \item Max penetration: Max penetration at each time step
    \item Dynamics index: Error from the ground-truth value for indices that reflect deformation, internal stress, etc.
    \item Computation time: Time for each step, divided into 1) dynamics time, to construct \eqref{eq-lindyn} that includes calculation of constraint error, constraint error Jacobian, mass matrix computation and collision detection, etc. and 2) solver time, to obtain $\hat{v}$ with the contact impulse $\lambda_c$.
\end{itemize}

As the ground-truth result for the dynamics index, real-world experimental values are used for two examples, while solutions from excessive iterations (thus, sufficiently converged with negligible residual) are used for the other examples.
For all scenarios, 20 repetitions are performed under several randomized parameter settings.

\subsubsection{Soft Mat Manipulation}

\begin{figure}[t] 
    \centering
    \subfigure[Soft mat manipulation simulation using COND]{
    \includegraphics[width=4.0cm]{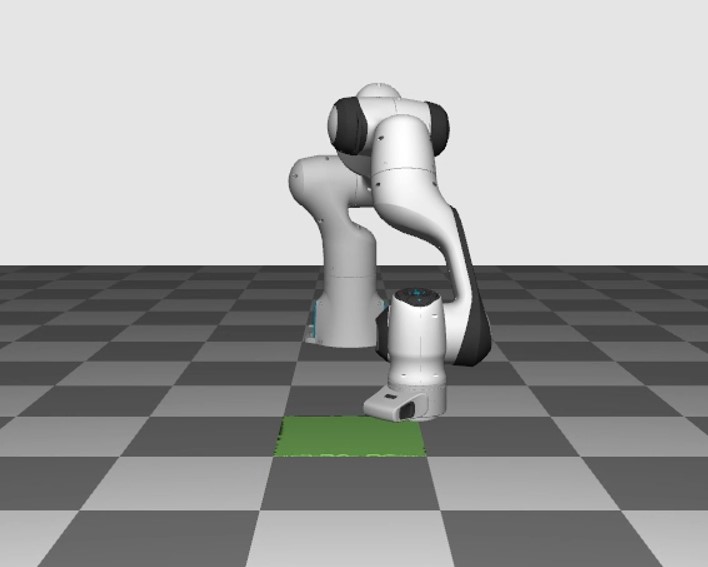}
    \includegraphics[width=4.0cm]{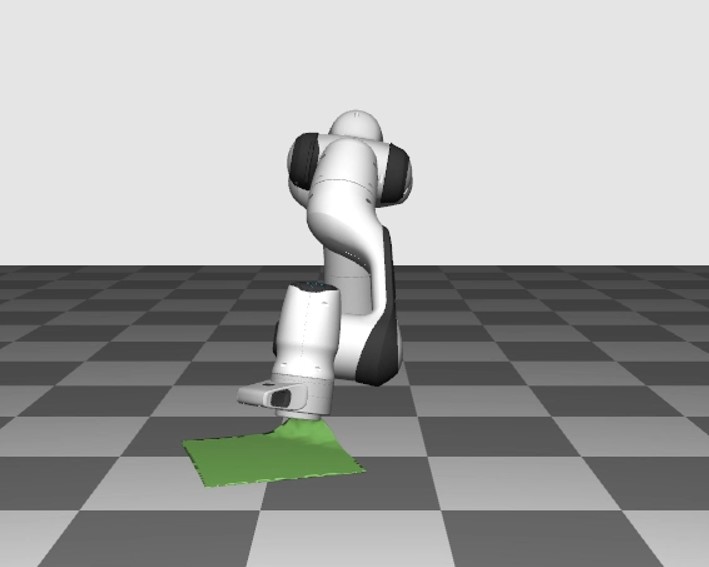}
    \label{subfig-matsnap}
    }
    \subfigure[Penetration depth]{
    \includegraphics[width=4.0cm]{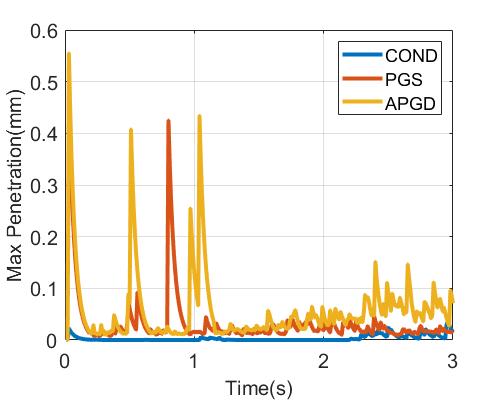}
    \label{subfig-matpenet}
    }
    \subfigure[Internal force]{
    \includegraphics[width=4.0cm]{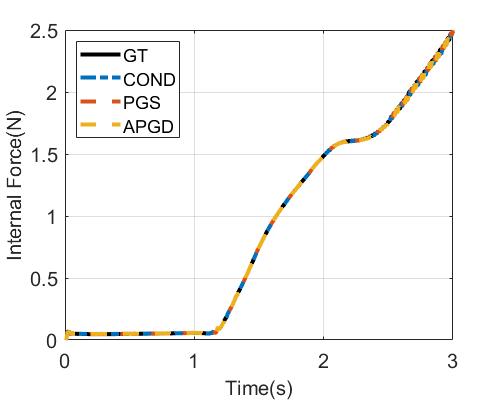}
    }
    \caption{Snapshots and performance comparison plots for soft mat folding manipulation simulation. GT: ground-truth.}
    \label{fig-matresult}
\end{figure}

\begin{table}[t]
\small
\centering
\renewcommand{\arraystretch}{1.5}{
\resizebox{8cm}{!}{
\begin{tabular}{|c|c|c|c|c|c|}
\hline
Solver & MP [mm] & IE [mN] & DT [ms] & ST [ms] & IN \\
\hline
\hline
COND & 0.0033 & 1.868 & 9.568 & 8.650 & 39.72 \\
\hline 
PGS & 0.0328 & 0.917 & 9.599 & 1412 (898) & 50.59 \\
\hline 
APGD & 0.0576 & 1.492 & 9.563 & 1645 (890) & 81.57 \\
\hline 
\end{tabular}
}
}   
\caption{Comparison results for soft mat folding manipulation simulation. Averaged values for 20 ramdomized trials and time steps. MP: max penetration, IE: internal force norm error, DT: dynamics time, ST: solver time, IN: iteration number. $(\cdot)$ computation time for Delassus operator assembly.}
\label{table-matresult}
\end{table}

Firstly, a robotic soft mat folding manipulation scenario is implemented. 
In this case, the contact situation includes a large number of contact and stick-slip transition behavior.
Franka Emika Panda is used to fold the soft mat, the co-rotational FEM model is used to formulate the deformable part dynamics, and parameters are randomly chosen between Young modulus $50-100~\rm{kPa}$, Poisson ratio $0.3-0.4$, and friction coefficient $0.1-0.5$.
Also in FEM modeling, 1477 nodes and 4100 tetrahedral elements are used, therefore total DOF is 4438 with 24606 constraints. 
For the dynamics index, we measure the internal force of the soft mat induced from its deformation and take the norm.
A residual threshold is set as $10^{-4}$.

\begin{figure}[!htb] 
    \centering
    \subfigure[Soft ball gripping simulation using COND]{
    \includegraphics[width=4.0cm]{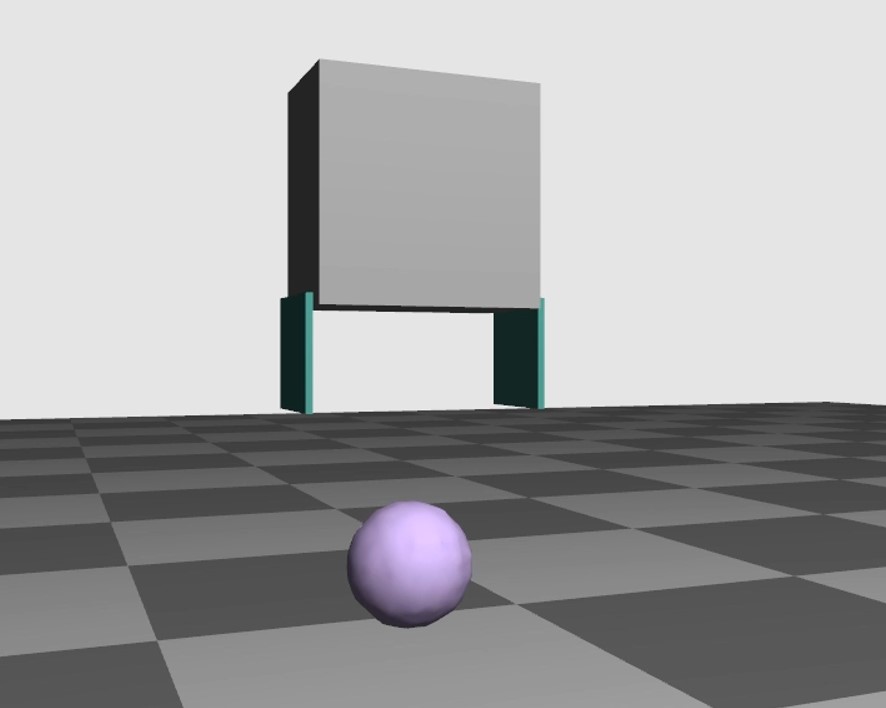}
    \includegraphics[width=4.0cm]{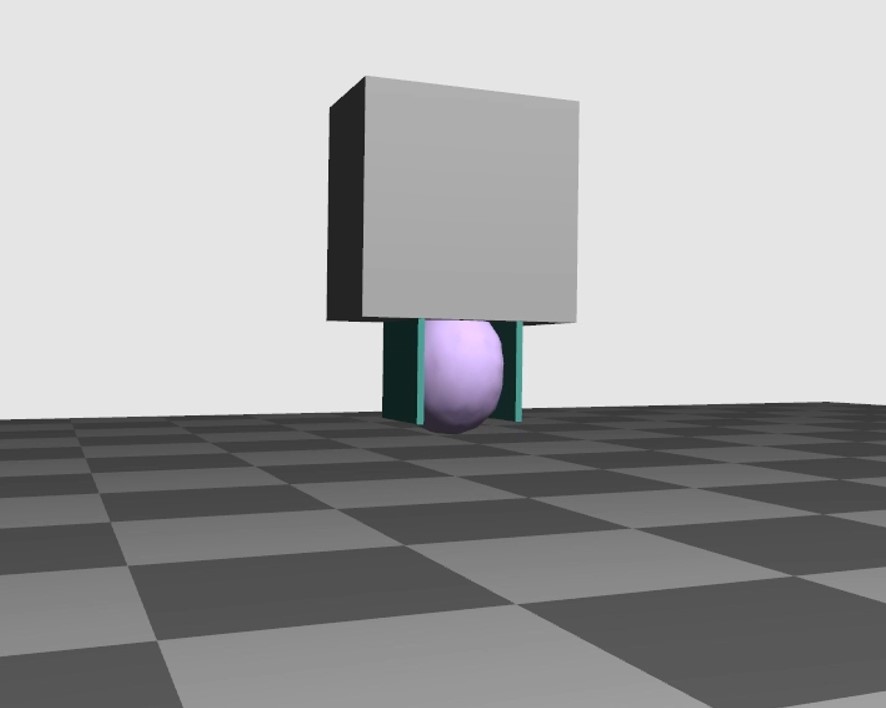}
    }
    \subfigure[Penetration depth]{
    \includegraphics[width=4.0cm]{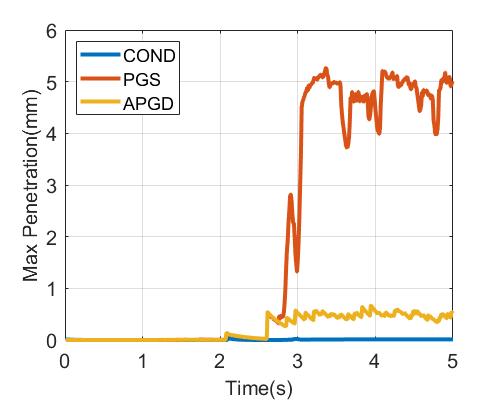}
    }
    \subfigure[Internal force]{
    \includegraphics[width=4.0cm]{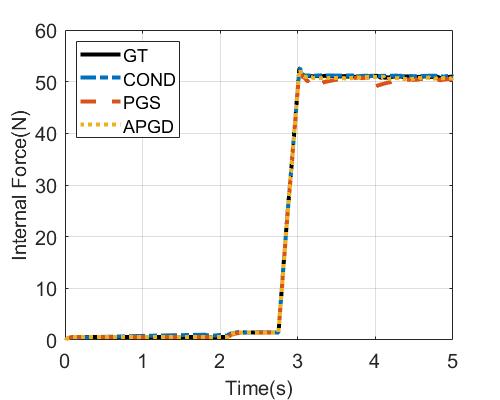}
    }
    \caption{Snapshots and performance comparison plots for soft ball gripping simulation. GT: ground-truth.}
    \label{fig-gripsoftresult}
\end{figure}

The results are illustrated in Fig.~\ref{fig-matresult} and Table~\ref{table-matresult}.
The most decisive difference is the computation time, where COND is over 100 times faster than PGS/APGD.
This is a result of the COND feature, which performs iteration only by sparse matrix-vector multiplication without dense matrix operation.
The methods based on I-FPI take near $900~\rm{ms}$ of Delassus operator calculation time alone.
The penetration depth in the 

\begin{table}[!htb]
\small
\centering
\renewcommand{\arraystretch}{1.5}{
\resizebox{8cm}{!}{
\begin{tabular}{|c|c|c|c|c|c|}
\hline
Solver & MP [mm] & IE [N] & DT [ms] & ST [ms] & IN \\
\hline
\hline
COND & 0.0058 & 0.1560 & 18.61 & 8.464 & 37.51 \\
\hline 
PGS & 1.9933 & 0.3189 & 18.51 & 805.1(624) & 233.9 \\
\hline 
APGD & 0.2325 & 0.1235 & 18.43 & 676.7(626) & 100.4 \\
\hline 
\end{tabular}
}
}   
\caption{Comparison results for soft ball gripping simulation. Averaged values for 20 ramdomized trials and time steps. MP: max penetration, IE: internal force norm error, DT: dynamics time, ST: solver time, IN: iteration number. $(\cdot)$ computation time for Delassus operator assembly.}
\label{table-gripsoftresult}
\end{table}

\begin{table}[!htb]
\small
\centering
\renewcommand{\arraystretch}{1.5}{
\resizebox{5cm}{!}{
\begin{tabular}{|c|c|c|c|c|c|}
\hline
Solver & 0.1 MPa & 1 MPa & 10 MPa \\
\hline
\hline
COND & 39.72 & 37.93 & 34.99  \\
\hline 
PGS & 145.8 & 224.0 & 306.7 \\
\hline 
APGD & 48.07 & 100.9 & 200.5 \\
\hline 
\end{tabular}
}
}   
\caption{Average number of iteration for gripping simulation of different stiffness soft balls.}
\label{table-gripsoftresult2}
\end{table}

\noindent COND result is much lower than in the PGS/APGD result.
This is very predictable because COND solves the surrogate dynamics problem at each iteration step so that the solution satisfies the contact condition.
On the other hand, COND shows a larger internal force error, which can be interpreted as the residual is biased to the dynamics error.
Note that the computation time of COND is fast, so this error can be further reduced by using more iterations.
For example, by lowering the residual threshold as $10^{-5}$, we can obtain IE $0.407~\rm{mN}$ and ST $18.90~\rm{ms}$. 
The average iteration numbers are comparable but lowest in COND, indicating that it converges to reasonable accuracy in tens of iterations on average.

\subsubsection{Soft Ball Gripping}

Next, the scenario where a rigid gripper grasps soft ball is implemented. 
In this case, gripping generates a combined contact situation between dynamically moving rigid and soft parts, which should be handled using virtual nodes in COND.
We design the gripper consisting of 6 DOF rigid bodies with 2 DOF prismatic joints which perform grasping of soft ball. 
A soft ball is modeled by co-rotational FEM and parameters are set as Young modulus $0.5-2~\rm{MPa}$, Poisson ratio $0.3-0.4$, and the friction coefficient is set to $0.4-0.6$.
Also in FEM modeling, 1463 nodes and 6928 tetrahedral elements are used, therefore total DOF is 4397 with 41568 constraints.
As in the soft mat scenario, the internal force of the soft ball is measured to assess the accuracy of dynamics, as it depends on the deformation and material properties of the ball. 
A residual threshold is set as $10^{-5}$.

The results are illustrated in Fig.~\ref{fig-gripsoftresult} and Table~\ref{table-gripsoftresult}, \ref{table-gripsoftresult2}.
Although the dynamics time process differs slightly from other solvers due to the use of a virtual node, the majority of the processes are the same, so the time required is very similar.
As in soft mat scenarios, the solver computation speed of COND is significantly faster than others (over x80).
Also, COND allows for the least amount of penetration here, demonstrating the validity of the virtual node formulation. 
APGD can successfully simulate gripping without significant penetration and has the lowest internal force error, while internal force error in COND is larger than APGD but comparable.
This is also a consistent result with the soft mat scenario, that the error of COND is dynamics-biased, but reaches a reasonable value after a few tens of iterations.

In many cases, PGS allows for large amounts of penetration.
We also observe this improper contact behavior eventually degrades overall dynamics accuracy, resulting in the greatest internal force error.
Failure of PGS is from convergence degradation of I-FPI, as the Delassus operator is ill-conditioned while the object is captured symmetrically on both sides.
Our further experimental results in Table~\ref{table-gripsoftresult2} shows that the PGS and APGD iteration numbers increase as the stiffness of the ball increases.
However, COND is resistant to this effect, as it takes V-FPI.

\subsubsection{Flexible Cable Manipulation}

\begin{figure}[t] 
    \centering
    \subfigure[Cable winding manipulation simulation using COND and experiment]{
    \includegraphics[width=4.0cm]{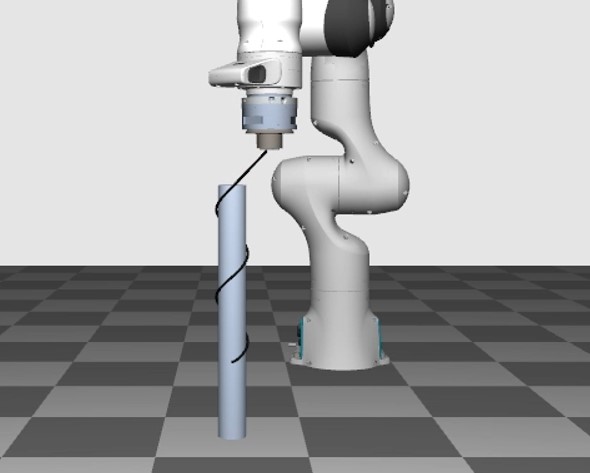}
    \includegraphics[width=4.0cm]{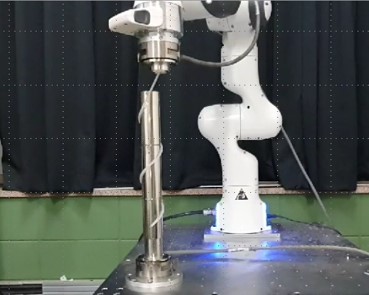}
    \label{fig-cablesnap}
    }
    \subfigure[Penetration depth]{
    \includegraphics[width=4.0cm]{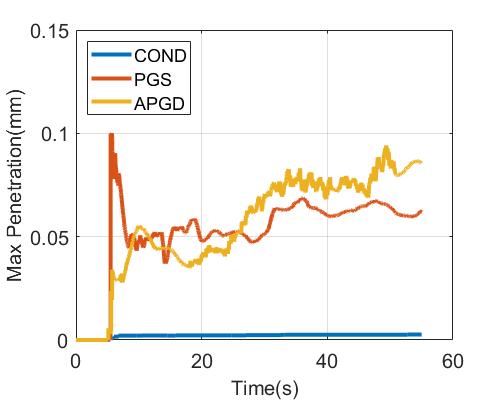}
    }
    \subfigure[End effector force ($x$)]{
    \includegraphics[width=4.0cm]{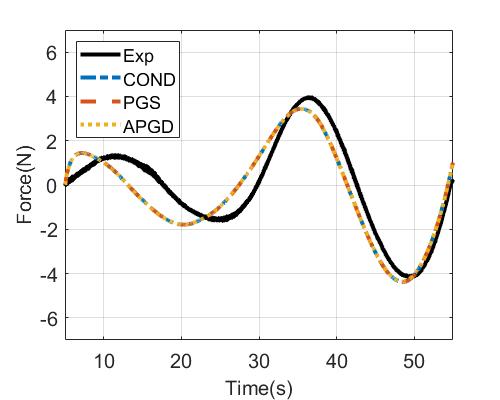}
    }
    \subfigure[End effector force ($y$)]{
    \includegraphics[width=4.0cm]{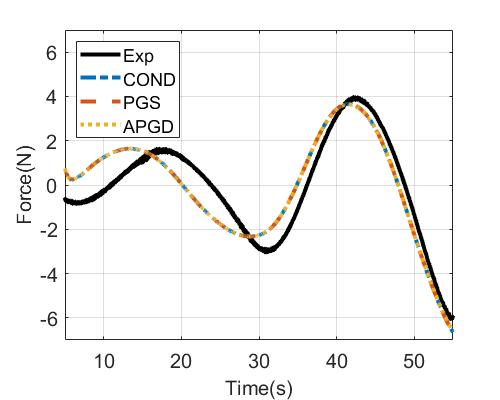}
    }
    \subfigure[End effector force ($z$)]{
    \includegraphics[width=4.0cm]{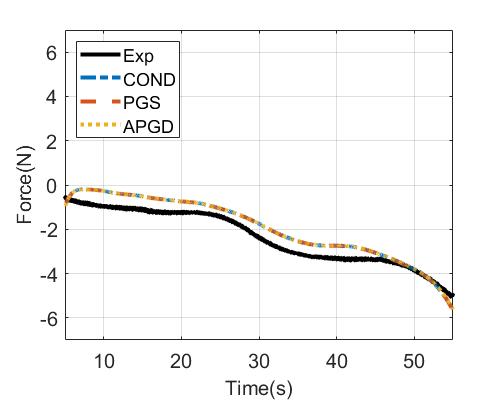}
    }
    \caption{Snapshots and comparison results for flexible cable winding manipulation simulation.}
    \label{fig-cableresult}
\end{figure}

As shown in Fig.~\ref{fig-cablesnap}, we conduct simulations and experiment on cable winding manipulation using a robot arm (Franka Emika Panda).
Here, dynamics equation of flexible cable is constructed using Cosserat rod model \cite{rubin13springer} and parameters of the cable are measured following to \cite{lee21icra}: Young modulus $2.954~\rm{MPa}$, Poisson ratio $0.49$, and friction coefficient $1.3$.
When testing for indicators other than experimental comparisons, these parameters are randomized as $4-5~\rm{MPa}$, $0.4-0.49$, and $0.5-1.5$.
The entire cable is split into a total of 320 links, therefore total DOF is 1927 with 1926 constraints.
For the comparison with the experiment, we measure the force applied to the end effector of the robot arm during the operation task using the F/T sensor (ATI Gamma).
A residual threshold is set as $10^{-4}$.

Overall results are depicted in Fig.~\ref{fig-cableresult} and Table~\ref{table-cableresult}.
As in previous scenarios, COND is significantly faster than the other two, while APGD outperforms PGS.
In the case of end effector force, the overall value agrees well with the experimental value, and the results of the three solvers are nearly identical.
This means that the COND iteration can rapidly arrive at a solution that satisfies the analytical cable model equations, which is close to the real world physics.
The remaining error can be described as a gap between the model and the real world, and it may be resolved through more precise matching (e.g., damping parameter) in future work.
Iteration number in COND is significantly lower than in others, and again, as expected, COND has the lowest penetration error.

\begin{table}[!htb]
\small
\centering
\renewcommand{\arraystretch}{1.5}{
\resizebox{8cm}{!}{
\begin{tabular}{|c|c|c|c|c|c|}
\hline
Solver & MP [mm] & FE [N] & DT [ms] & ST [ms] & IN \\
\hline
\hline
COND & 0.0021 & 0.3996 & 2.086 & 1.229 & 18.28 \\
\hline 
PGS & 0.0521 & 0.3997 & 2.103 & 74.27 (28.2) & 65.61 \\
\hline 
APGD & 0.0556 & 0.3976 & 2.112 & 56.72 (27.9) & 46.05 \\
\hline 
\end{tabular}
}
}   
\caption{Comparison results for flexible cable winding simulation. Averaged values for 20 ramdomized trials and time steps except FE. MP: max penetration, FE: end effector force norm error, DT: dynamics time, ST: solver time, IN: iteration number. $(\cdot)$ computation time for Delassus operator assembly.} 
\label{table-cableresult}
\end{table}

\subsubsection{Soft Gripper}

\begin{figure}[t] 
    \centering
    \subfigure[Soft gripper manipulation simulation using COND and experiment.]{
    \includegraphics[width=4.0cm]{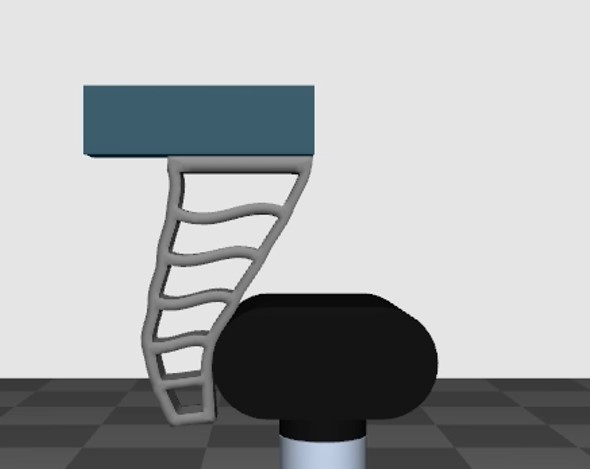}
    \includegraphics[width=4.0cm]{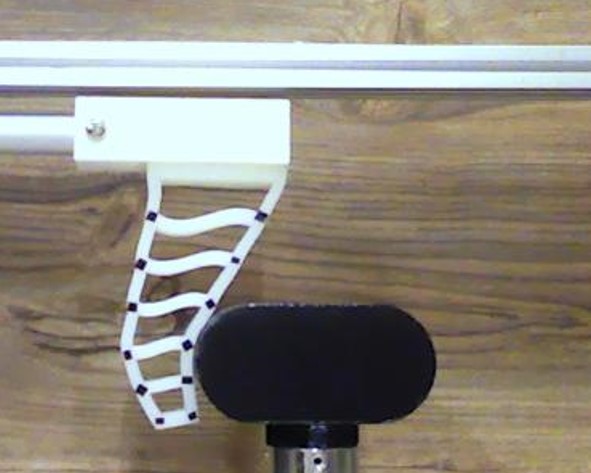}
    \label{fin-a}
    }
    \subfigure[Penetration depth]{
    \includegraphics[width=4.0cm]{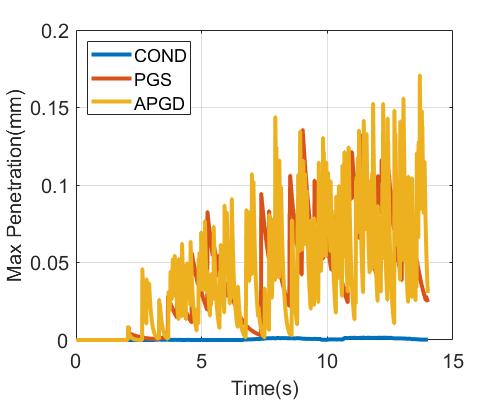}
    }
    \subfigure[Contact force ($x$)]{
    \includegraphics[width=4.0cm]{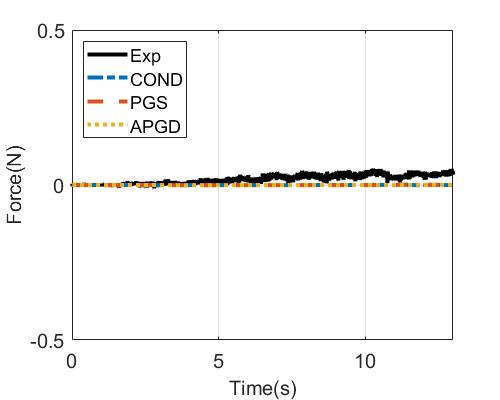}
    }
    \subfigure[Contact force ($y$)]{
    \includegraphics[width=4.0cm]{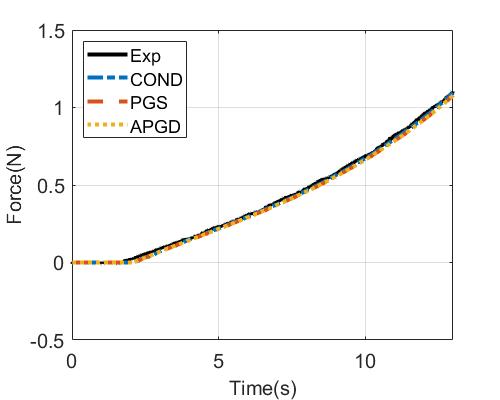}
    }
    \subfigure[Contact force ($z$)]{
    \includegraphics[width=4.0cm]{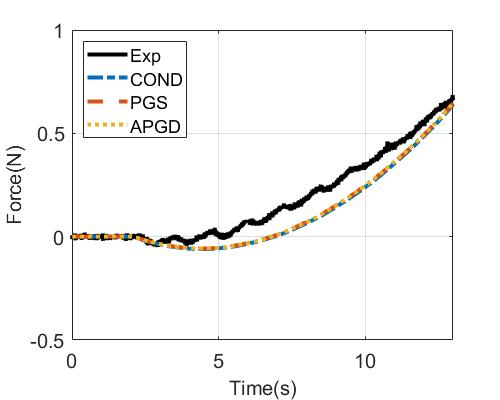}
    }
    \caption{Snapshots and comparison results for soft gripper manipulation simulation.}
    \label{fig-finresult}
\end{figure}

\begin{table}[t]
\small
\centering
\renewcommand{\arraystretch}{1.5}{
\resizebox{8.8cm}{!}{
\begin{tabular}{|c|c|c|c|c|c|c|}
\hline
Solver & MP [mm] & FE [N] & DE [mm] & DT [ms] & ST [ms] & IN \\
\hline
\hline
COND & 0.0005 & 0.1169 & 1.898 & 43.24 & 18.24 & 30.88 \\
\hline 
PGS & 0.0397 & 0.1175 & 1.895 & 42.89 & 547.0 (533) & 59.06 \\
\hline 
APGD & 0.0409 & 0.1172 & 1.896 & 42.78 & 547.3 (533) & 57.81 \\
\hline 
\end{tabular}
}
}   
\caption{Comparison results for soft gripper manipulation simulation. Averaged values for 20 ramdomized trials and time steps except . MP: max penetration, FE: contact force norm error, DE: displacement error at final state, DT: dynamics time, ST: solver time, IN: iteration number. $(\cdot)$ computation time for Delassus operator assembly.} 
\label{table-finresult}
\end{table}

Lastly, we simulate and experiment with the contact situation between the soft gripper and the rigid object.
The soft gripper part is made using liquid silicone Ecoflex 0050 (Smooth-on Inc.) and attached to a linear actuator so that it could be operated as a gripper.
The rigid object part is manufactured using a 3D printer (PLA).
The soft gripper dynamics is modeled using co-rotational FEM, and we use the Ecoflex 0050 material parameter reported in \cite{roberts13icra} (Young modulus: $83~\rm{kPa}$, Poisson ratio: $0.42$), and the friction coefficient is determined using a simple stick-slip test, yielding a value of $2.1$.
When testing for indicators other than experimental comparisons, these parameters are randomized
as $50-100~\rm{kPa}$, $0.4-0.49$, and $0.5-2.1$. 
For FEM modeling, 4889 nodes and 18122 tetrahedral elements are used, therefore total DOF is 14681 with 108744 constraints.
The rigid part is modeled as a dynamic object rather than a static environment, so virtual nodes are applied to its points of contact.
For experimental comparison, the contact force is measured with the F/T sensor (ATI Gamma) connected to the rigid part, and the displacement of the markers attached to the soft gripper at the final state is measured using a vision sensor (SC-FD110B).
A residual threshold is set at $10^{-6}$.

Fig.~\ref{fig-finresult} and Table~\ref{table-finresult} show the overall result.
As in cable manipulation scenario, the force/displacement results well match with the experiment for all three solvers, and the difference in error level is insignificant.
But still, there exists a gap with reality.
COND allows for the least amount of penetration and the fastest solver times (over 30x faster).
Note that in the scenario, the number of contact points is relatively small (around 40), therefore the speed increase amount in COND is relatively lower compare to other scenarios.
Finally, COND shows the lowest average number of iterations.
\subsection{Scalability}

\begin{figure}[t] 
    \centering
    \subfigure[COND]{\includegraphics[width=4.2cm]{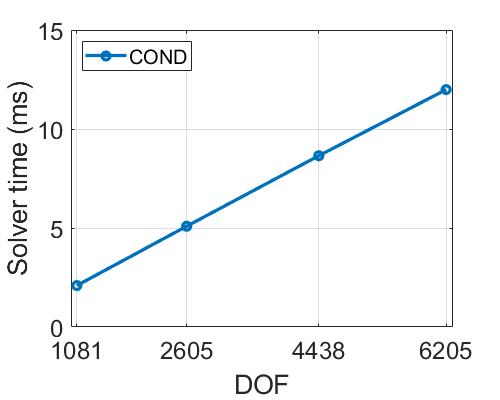}}
    \subfigure[PGS/APGD]{\includegraphics[width=4.2cm]{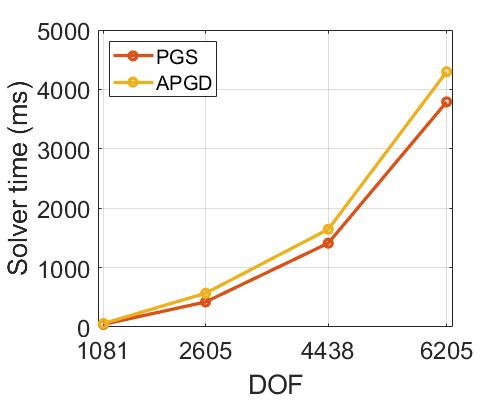}}
    \caption{Computation time results for soft mat manpulation scenario with various number of node.}
    \label{fig-scalability}
\end{figure}

Although the computational advantage of COND is clearly demonstrated in the preceding subsection, we proceed with further experiments to measure its scalability in more detail.
For the soft mat manipulation scenario, we check the solver time while changing the node number of the FEM model.
A residual threshold is set at $10^{-4}$ for all tests.

Scalability measure results are depicted in Fig.~\ref{fig-scalability}.
The computation time of COND increases as almost linear (R-squared value: 0.9997) as DOF increases, which demonstrates the analysis in Sec.~\ref{subsec-complexity}.
Meanwhile, computation of PGS/APGD grows super-linearly as DOF increases, which originated from matrix factorization and dense matrix operations.

\subsection{Comparison with Other Methods in Graphics}

Apart from popular and versatile I-FPI-based methods, there have been several attempts to improve computational efficiency.
Most of them are originated in the field of graphics, primarily dealing with large-scale problems, and in this subsection, COND and their comparisons are discussed.

\subsubsection{Contact constraint splitting}

\begin{table}[t]
\small
\centering
\renewcommand{\arraystretch}{1.5}{
\resizebox{6cm}{!}{
\begin{tabular}{|c|c|c|c|c|c|}
\hline
Solver & MP [mm] & IE [mN] & ST [ms]\\
\hline
\hline
COND (50) & 0.0049 & 0.9662 & 10.85  \\
\hline 
COND (100) & 0.0050 & 0.1620 & 20.90  \\
\hline 
CCS (50) & 0.0051 & 13.40 & 65.99  \\
\hline 
CCS (100) & 0.0053 & 6.559 & 107.9  \\
\hline 
\end{tabular}
}
}   
\caption{Comparison results of COND and CCS for soft mat manipulation simulation. 
Averaged values for 20 ramdomized trials and time steps. 
MP: max penetration, IE: internal force norm error, ST: solver time. 
$(\cdot)$ iteration number.}
\label{table-ccsresult}
\end{table}

In \cite{daviet20tog}, the contact dynamics problem is divided into two subproblems using the alternating direction method of multiplier (ADMM), thereby replacing the assembly of the Delassus operator with multiple linear solving processes.
The methods significantly reduce simulation time for hair/cloth, making them state-of-the-art in the field.
In this paper, we denote the algorithm as contact constraint splitting (CCS) and implement it for a soft mat manipulation scenario.
Note that the method is intended for thin nodal objects, thus it is not compatible with contact situations over joint variables or rotation representations (e.g., gripping scenarios).
Since CCS employs ADMM, it is ambiguous to set the fair residual threshold value because primal and dual residuals must be considered together.
Hence, we compare the performance under the fixed number of iterations.

The results are summarized in Table~\ref{table-ccsresult}.
As the projection variable of CCS guarantees the contact feasibility, it also shows a comparable penetration amount with COND.
However, internal force error result is much similar in COND, which means COND converges faster to accurate solution.
Also, computation time is nearly 5 times faster in COND.
This is unsurprising, because CCS requires linear solving of size $n$ for every iteration step, whereas COND only necessitates matrix-vector multiplication.
Overall, given the versatility of COND for robotic simulation, it is clear that this is a better option in many cases.
Combining CCS with contact nodalization to improve its applicability (e.g., joint variables) is also an interesting option, but we observe that the ADMM algorithm of CCS does not converge well in this case.

\subsubsection{Position-based dynamics} 

\begin{figure}[t] 
    \centering
    \subfigure[Simulation snapshot (PBD 100)]{
    \includegraphics[width=4.0cm]{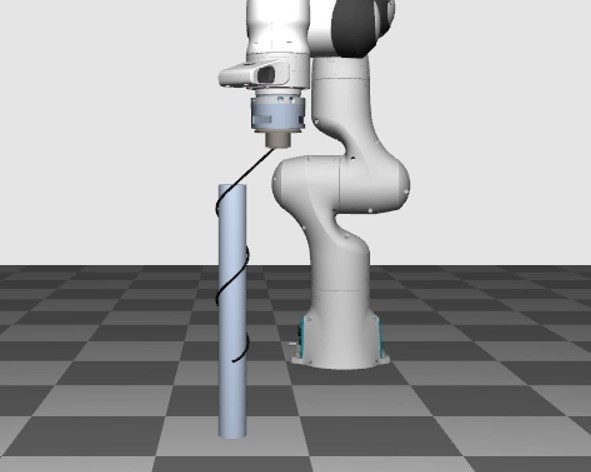}
    \label{fig-pbdsnapshot}
    }
    \subfigure[End effector force ($x$)]{
    \includegraphics[width=4.0cm]{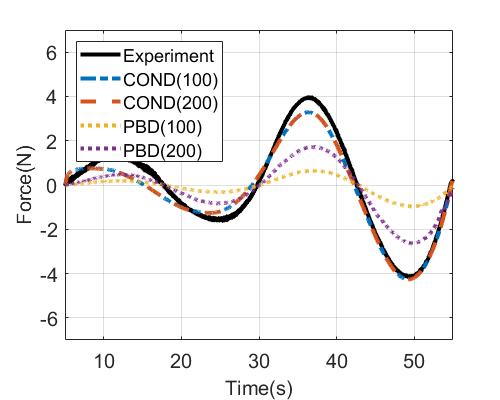}
    }
    \subfigure[End effector force ($y$)]{
    \includegraphics[width=4.0cm]{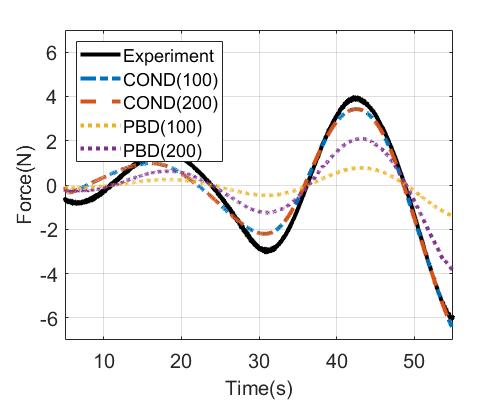}
    }
    \subfigure[End effector force ($z$)]{
    \includegraphics[width=4.0cm]{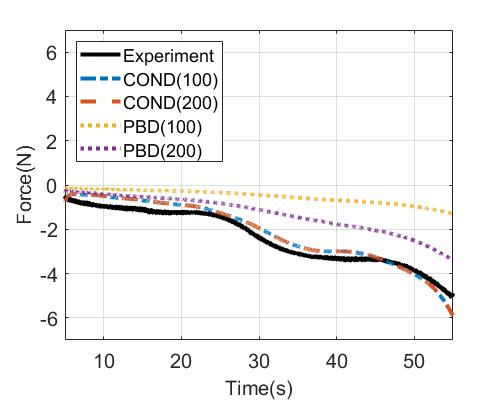}
    }
    \caption{Snapshot and end effector force plots of COND and PBD for flexible cable winding
manipulation simulation,}
    \label{fig-pbdresult}
\end{figure}

\begin{table}[t]
\small
\centering
\renewcommand{\arraystretch}{1.5}{
\resizebox{6cm}{!}{
\begin{tabular}{|c|c|c|c|c|c|}
\hline
Solver & MP [mm] & FE [N] & TT [ms]\\
\hline
\hline
COND (100) & 0.0031 & 0.553 & 3.949 \\
\hline 
COND (200) & 0.0032 & 0.553 & 7.125 \\
\hline 
PBD (100) & 0.0036 & 2.654 & 94.45  \\
\hline 
PBD (200) & 0.0026 & 1.575 & 187.6  \\
\hline 
\end{tabular}
}
}   
\caption{Comparison results of COND and PBD for flexible cable winding simulation. 
Averaged values for 20 ramdomized trials and time steps except FE. 
MP: max penetration, FE: end effector force norm error, TT: total computation time. $(\cdot)$ iteration (substep) number.}
\label{table-pbdresult}
\end{table}

For various graphical and robotic applications \cite{liu21icra,lie22arxiv} with open-source software (FleX, Brax, etc.), position-based dynamics (PBD) is a prevalent method for deformable body simulation.
In PBD, non-linear Gauss-Seidel iteration is performed for each constraint, therefore, can avoid large-size matrix factorization and multiplication.
We implement the flexible cable manipulation scenario using position-based dynamics (PBD).
We adopt the substepping and contact handling method presented in \cite{macklin19sca}, which are the state-of-the-art techniques to increase the performance of PBD.
PBD differs from our solver in several respects, as it is not derived from the form of \eqref{eq-lindyn}, nor does it use complementarity-based contact conditions.
Thus, as in CCS comparison, performance over a fixed iteration number is evaluated.
When using 320 links as in Sec.~\ref{subsec-multibodyexp}, we observe that the convergence of the PBD is extremely poor. Instead, 120 links are used.

The results are depicted in Table~\ref{table-pbdresult}.
For COND, the results from 100 and 200 iterations are nearly identical, while the force error matches the experimental results well.
However, in PBD results, the result varies with iteration number, which denotes that the solution does not sufficiently converge to the ground-truth solution.
Also, as shown in Fig.~\ref{fig-pbdsnapshot}, the visual appearance of the PBD result is plausible, implying that the results for proper graphics may be insufficient for overall accuracy.
In terms of computation time, COND has a faster computation time per iteration step.
This is due to the fact that PBD performs sequential constraint handling with state update, which is more time-consuming than a single matrix-vector multiplication.

\subsection{Invertible Contact Model} \label{subsec-invertible}

\begin{figure}[t] 
    \centering
    \subfigure[Norm of contact force applied to the mat during Franka soft mat folding]{
    \includegraphics[width=4.2cm]{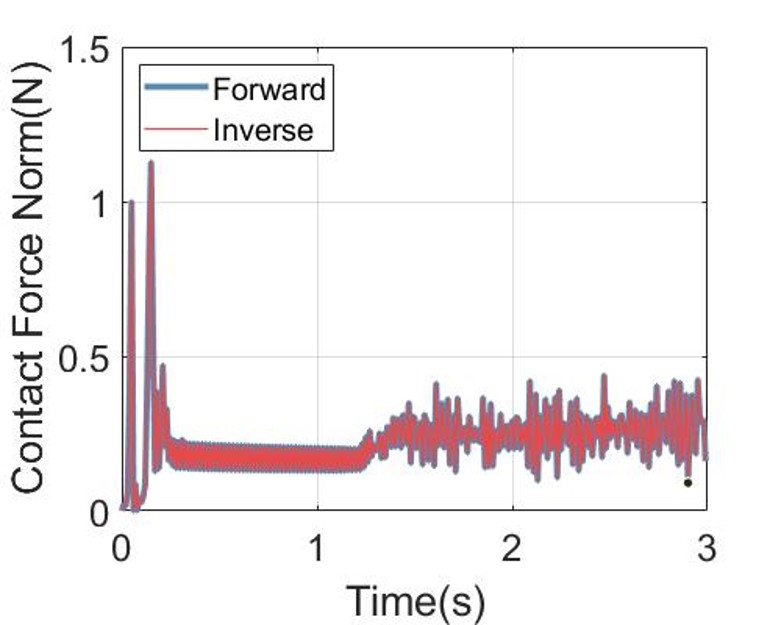}
    \includegraphics[width=4.2cm]{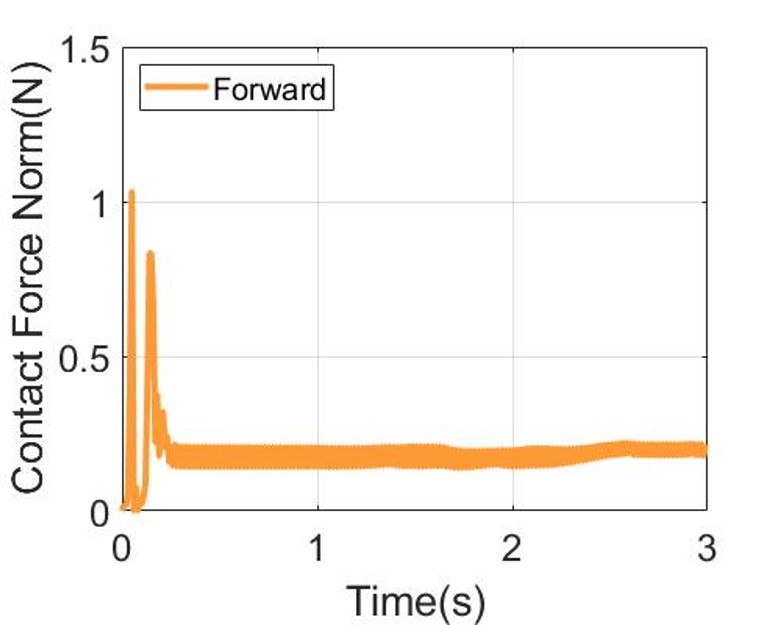}
    }
    \subfigure[Norm of contact force applied to the soft ball during gripping]{
    \includegraphics[width=4.2cm]{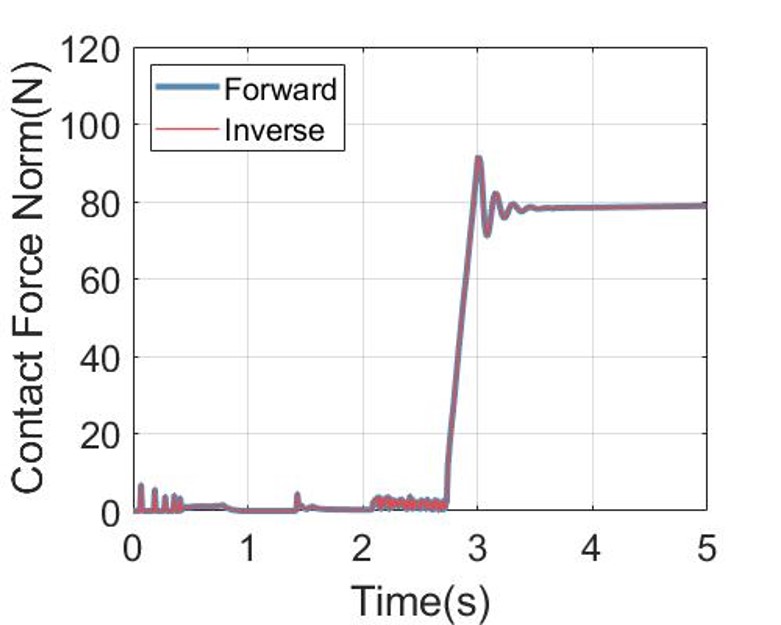}
    \includegraphics[width=4.2cm]{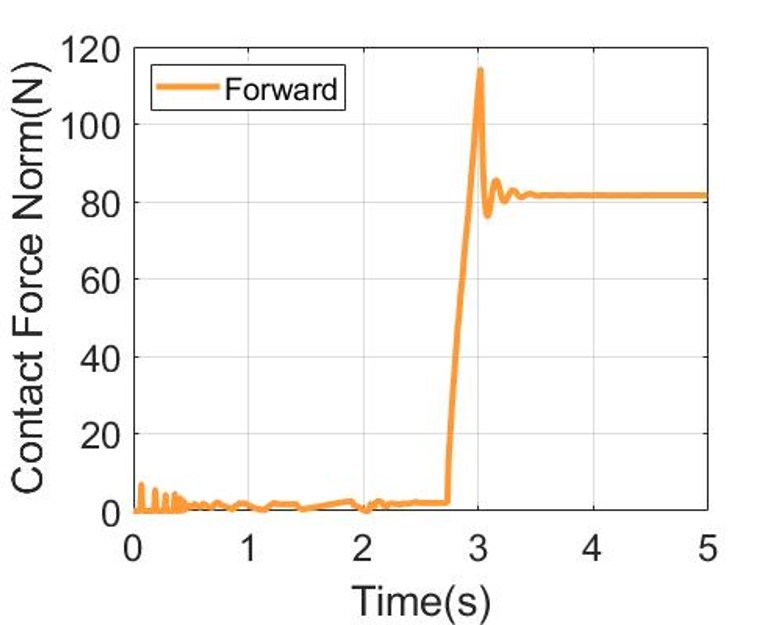}
    }
    \caption{Force results from two scenarios. Left: proximal operator, blue line denotes contact impulse norm computed using forward dynamics and red line denotes contact impulse norm computed inversely from the result. Right: strict operator.}
    \label{fig-invtresult}
\end{figure}

The applicability of COND in the invertible contact model is evaluated.
For soft mat and soft ball gripping scenarios, we applied the model described in Sec.~\ref{subsec-invt} and record vector norm of forwardly/inversely calculated contact forces applied to all nodes for each time step. 
The implementation is based on the proximal operator, and regularization term is set as $\Omega_c = 10^{-3} I$ for both scenarios, with residual threshold $10^{-4}$.

Results are depicted in Fig.~\ref{fig-invtresult}. 
We can find that the invertible property is well preserved as a forward result and the inverse result well-matched, as the error amount is $0.087~\rm{mN}$ for soft mat, $0.48~\rm{mN}$ for soft ball.
This demonstrates the solver successfully handle the convex optimization problem. 
In Franka soft mat scenario, the results using the proximal operator show that jittering occurs during sliding unlike the results using the strict operator. 
This seems to have occurred due to the fact that Signorini conditions are not exactly satisfied in the proximal operator, and repeatedly take off and collide simultaneously across nodes. 
On the other hand, for the softball gripping scenario, the two results are very similar. 
Also in some cases, the proximal operator requires an adequate amount of $\phi$ to avoid implausible contact behavior (e.g., detaching).

\subsection{Anisotropic Friction}

\begin{figure}[t] 
    \centering
    \subfigure[Experiment setup]{
    \includegraphics[width=4.0cm]{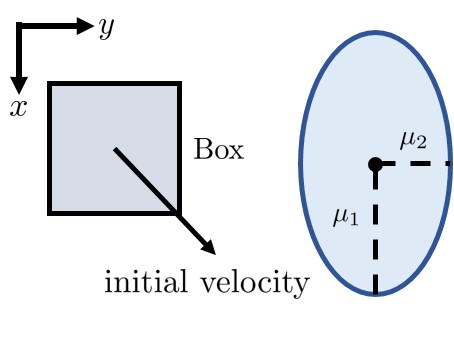}
    \label{fig-anisotropic1}
    }
    \subfigure[Box trajectory]{
    \includegraphics[width=4.0cm]{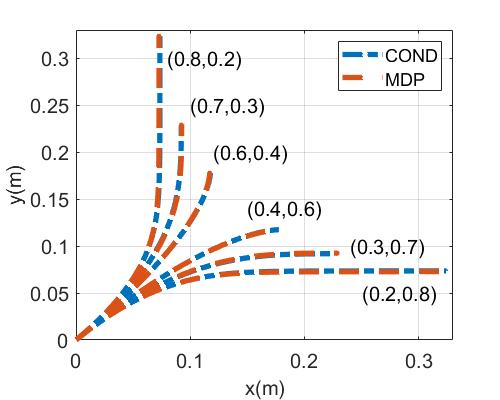}
    \label{fig-anisotropic2}
    }
    \caption{Box sliding simulation on anisotropic friction region. Left: snapshot, red coordinates visualize anisotropic friction model, Right: position plot.}
    \label{fig-anisotropic}
\end{figure}

Simulation with anisotropic friction is tested using box slipping scenario. 
As in Fig.~\ref{fig-anisotropic2}, we set ellipsoidal friction model \eqref{eq-genfc} with various set of $(\mu_1,\mu_2)$ and measure trajectory while the initial velocity is set to $[1;1;0]$.
Then we compare the result with the solution obtained from MDP problem, and the result is depicted in Fig.~\ref{fig-anisotropic2}.
The trajectory of the box is curved, as seen in other experimental reports \cite{ma18ral}.
In addition, the trajectory is exactly matched with the MDP solution, demonstrating the property described in Prop. 6.

\subsection{Effect of Chebyshev Acceleration}

\begin{table}[t]
\small
\centering
\renewcommand{\arraystretch}{1.5}{
\resizebox{6cm}{!}{
\begin{tabular}{|c|c|c|c|}
\hline
Scenario &  \makecell{Avg. iter.\\(w.o. Cheby.)} & \makecell{Avg. iter.\\(w. Cheby.)} \\
\hline
\hline
Franka mat & 167.4 & 39.72 \\
\hline 
Gripper ball & 138.8 & 37.51 \\
\hline 
Franka cable & 40.14 & 18.28 \\
\hline 
Soft gripper & 155.7 & 30.88 \\
\hline 
\end{tabular}
}
}   
\caption{Effect of Chebyshev acceleration in iteration number.}
\label{cheby_table}
\end{table}

To evaluate the efficacy of Chebyshev acceleration in COND, we perform an ablation study that compares the average iteration number for convergence with/without Chebyshev acceleration.
As shown in Table~\ref{cheby_table}, the average iteration number to reach the threshold value is decreased for all scenarios and demonstrates that Chebyshev acceleration effectively works on COND.

\section{Discussions and Future Work} \label{sec-discussion}

While COND handles common multibody situations, better alternatives exist in some cases.
For example, in granular material simulation or tight tolerance assembly (e.g., bolt-nut assembly) simulation, computation of the Delassus operator is easy, while COND requires a large number of virtual nodes.
In such cases, I-FPI-based methods such as \cite{mazhar15tog} may perform better.
In general, COND is most effective when the system involves many internal constraints (therefore, hard to factorize $A$), such as deformable object manipulation.
Therefore, the direction of proper integration of CONDs with other solvers will be an interesting topic.

Although not explicitly demonstrated through implementation, we believe that other kinematic constraints such as non-holonomic constraints can be dealt with in following approaches:
1) softening and incorporation as a potential function; and 2) extension of contact diagonalization into a general constraint form using the idea of virtual nodes.
In future work, we will concretely develop the approaches, with several implemented scenarios.

The current strategy for determining $W$ in Sec.~\ref{subsec-SSM} works well, but there may still be a room for improvement.
For instance, Nesterov momentum \cite{mazhar15tog} and Anderson acceleration \cite{zhang19tog} can be adopted to our framework.
The Barzilai-Borwein method including diagonal scaling \cite{park20icassp} is also seen as a possible direction.
Based on the strategies, verifying the theoretical convergence rate will be useful.
Also, the number of fixed points remains unknown when using the strict operator.
While it does not yet induce practical issues, it makes sense to explore cases where solutions do not exist.
It will also be interesting to present other operators capable of ensuring uniqueness and other useful properties.

As shown in comparison with experiments, simulators can represent reality to some extent, but there are still sim-to-real gaps for a variety of reasons.
From this point of view, integration with the so-called real-to-sim method such as \cite{liu21icra} will be promising.
Further improvement of efficiency by combining with model order reduction schemes is a practically meaningful direction.
As our method is well compatible with parallelization, GPU implementation will contribute well to the performance.
Finally, as an extension of our project in https://github.com/INRoL/inrol\_sim\_cablewinding, the development of an open-source framework will be a valuable contribution to the community.

\section{Conclusion} \label{sec-conclusion}

In this paper, we propose the new multibody simulation framework COND. 
The framework mainly focuses on developing velocity fixed-point iteration, which can avoid large-dimensional matrix factorization and multiplication while utilizing a complementarity-based contact model.
To that end, we first propose contact nodalization based on virtual nodes, which converts all contact into nodal situations.
Then, using the contact diagonalization technique, we create a contact solving algorithm based on solving multiple surrogate dynamics problems, and each surrogate dynamics problem can be solved in a one-shot/parallelized manner.
Theoretical statements related to the accuracy and convergence of the solver are presented.  
Simulations and experiments are carried out for a variety of multibody examples.
The results demonstrate that our solver is significantly faster than popular factorization-based solvers, and shows near $\mathcal{O}(n)$ complexity in practice.
Also, it is shown that COND can produce the result with convincing accuracy compared to the ground-truth results.
Despite its performance and versatility, COND still has room to evolve such as numerical techniques and integration with other solvers, and future work will focus on topic like these.


\bibliographystyle{unsrt}
\bibliography{TRO_COND}

\end{document}